\documentclass[letterpaper]{article} 
\usepackage{aaai24}  
\usepackage{times}  
\usepackage{helvet}  
\usepackage{courier}  
\usepackage[hyphens]{url}  
\usepackage{graphicx} 
\usepackage{natbib}  
\usepackage{caption} 
\usepackage{algorithm}

\usepackage{color}
\usepackage{amsfonts,amsmath,amssymb,amsthm}
\newcommand\blue{\textcolor{blue}}

\newcommand\red{\textcolor{black}}
\usepackage{float}
\newtheorem{theorem}{Theorem}[section]
\newtheorem{assumption}[theorem]{Assumption}
\newtheorem{proposition}[theorem]{Proposition}
\usepackage{algpseudocode}
\usepackage{newfloat}
\usepackage{listings}
\DeclareCaptionStyle{ruled}{labelfont=normalfont,labelsep=colon,strut=off} 
\floatstyle{ruled}
\newfloat{listing}{tb}{lst}{}
\floatname{listing}{Listing}
\title{Semantic-aware Data Augmentation for Text-to-image Synthesis}
\author{
    Zhaorui Tan\textsuperscript{\rm 1,}\textsuperscript{\rm 2}, Xi Yang\textsuperscript{\rm 1}$^*$, Kaizhu Huang\textsuperscript{\rm 3}\thanks{Corresponding authors}
}
\affiliations{
    \textsuperscript{\rm 1}Department of Intelligent Science, Xi’an Jiaotong-Liverpool University\\
    \textsuperscript{\rm 2}Department of Computer Science, University of Liverpool \\
    \textsuperscript{\rm 3} Data Science Research Center, Duke Kunshan University \\
    Zhaorui.Tan21@student.xjtlu.edu.cn,
    Xi.Yang01@xjtlu.edu.cn,
    kaizhu.huang@dukekunshan.edu.cn
}

\usepackage{bibentry}

\begin{document}

\maketitle

\begin{abstract}
Data augmentation has been recently leveraged as an effective regularizer in various vision-language deep neural networks. However, in text-to-image synthesis (T2Isyn), current augmentation wisdom still suffers from the semantic mismatch between augmented paired data. Even worse, semantic collapse may occur when generated images are less semantically constrained. In this paper, we develop a novel Semantic-aware Data Augmentation (SADA) framework dedicated to T2Isyn. In particular, we propose to augment texts in the semantic space via an Implicit Textual Semantic Preserving Augmentation ($ITA$), in conjunction with a specifically designed Image Semantic Regularization Loss ($L_r$) as Generated Image Semantic Conservation, to cope well with semantic mismatch and collapse.
As one major contribution, we theoretically show that $ITA$ can certify better text-image consistency while $L_r$ regularizing the semantics of generated images would avoid semantic collapse and enhance image quality.
Extensive experiments validate that SADA enhances text-image consistency and improves image quality significantly in T2Isyn models across various backbones. Especially, incorporating SADA during the tuning process of Stable Diffusion models also yields performance improvements.

\end{abstract}

\section{Introduction}
Text-to-image synthesis (T2Isyn) is one mainstream task in the visual-language learning community that has yielded tremendous results. Image and text augmentations are two popular methods for regularizing visual-language models~\cite{naveed2021survey,liu2020survey}. 
As shown in Figure~\ref{fig:training_para}~(a),
existing T2Isyn backbones~\cite {xu2018attngan,tao2022df,wang2022clip}  typically concatenate noises to textual embeddings as the primary text augmentation method~\cite{reed2016generative} whilst employing simply basic image augmentations (\textit{e.g,}, Crop, Flip) on images' raw space. 
Recent studies~\cite{dong2017i2t2i,cheng2020rifegan} suggest text augmentation to be more critical and robust than image augmentation for T2Isyn, given that real texts and their augmentations involve the inference process. 

\begin{figure}
\centering
\includegraphics[width=\linewidth]{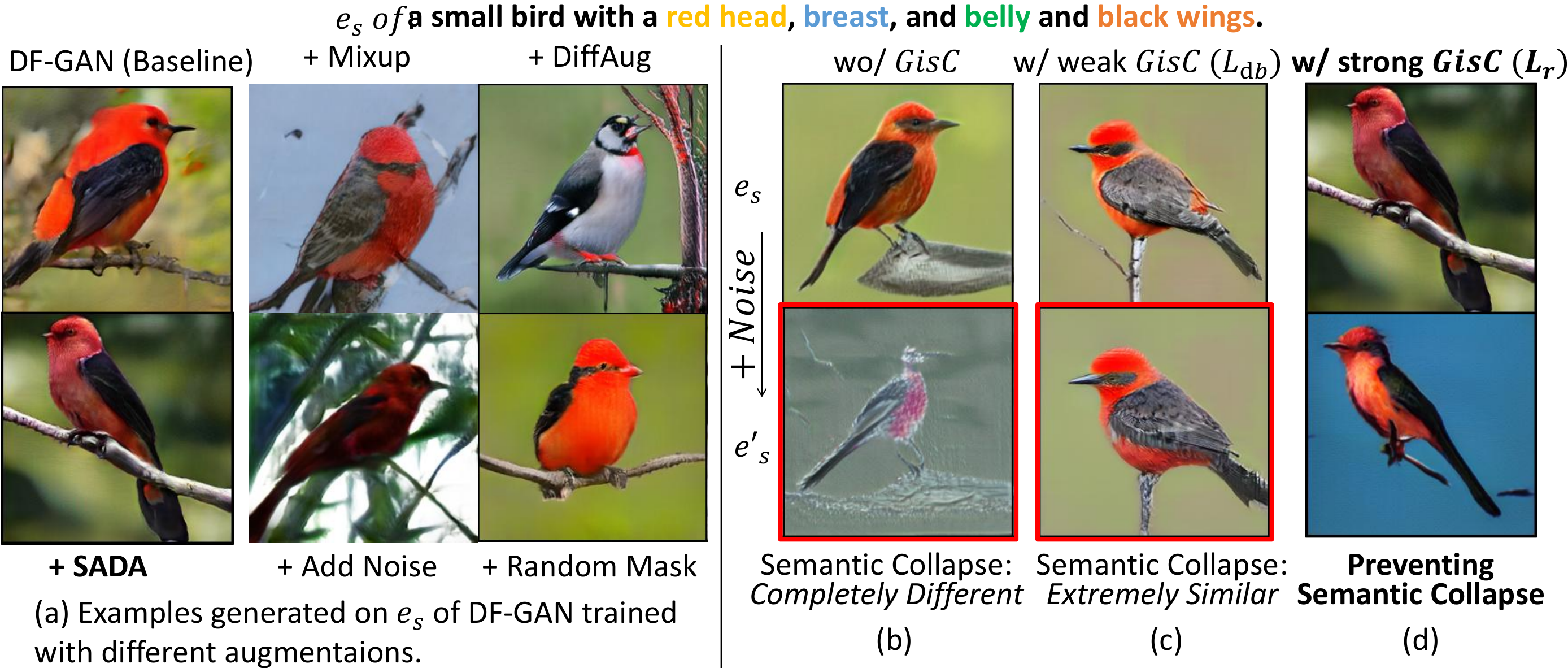}
\caption{(a) Current augmentations cause semantic mismatch and quality degradation in T2Isyn task. (b)(c) Illustrations of semantic collapse. (d) Our method prevents semantic collapse. See Supplementary Materials~\ref{app:More_Results} for more.}
\label{fig:t_banner}
\end{figure}

Albeit their effectiveness, we argue that current popular augmentation methods exhibit two major limitations in the T2Isyn task: 
1) 
\textit{Semantic mismatch} exists between augmented texts/images and generated pairs, it triggers accompanied semantic distribution disruption across both modalities, leading to augmented texts/images lacking corresponding visual/textual representations. 
As shown in Figure~\ref{fig:t_banner}~(a), advanced image augmentation, such as Mixup~\cite{zhang2017mixup}, {DiffAug~\cite{zhao2020differentiable}}, along with text augmentation like Random Mask\footnote{Randomly masking words in raw texts.} or Add Noise\footnote{Directly adding random noise to textual semantic embeddings.} might weaken both semantic and visual supervision from real images.
2) \textit{Semantic collapse} occurs in the generation process, \textit{i.e.}, when two \textit{slightly} semantic distinct textual embeddings are given, the model may generate either \textit{completely} different or \textit{extremely} similar images. This indicates that the models may be under-fitting or over-fitting semantically (see Figure~\ref{fig:t_banner}~(b)(c)).  
Both issues will compromise semantic consistency and generation quality. While imposing semantic constraints on generated images can alleviate semantic collapse, the study~\cite{wang2022clip} solely focuses on regulating the direction of semantic shift, which may not be entirely adequate.


Motivated by these findings, this paper proposes a novel Semantic-aware Data Augmentation (SADA) framework that offers semantic preservation of texts and images. SADA consists of an Implicit Textual Semantic Preserving Augmentation ($ITA$) and a Generated Image Semantic Conservation ($GisC$).
$ITA$ efficiently augments textual data and alleviates the semantic mismatch; $GisC$ preserves generated image semantics distribution by adopting constraints on semantic shifts.
As one major contribution, we show that SADA can both certify better text-image consistency and avoid semantic collapse with a theoretical guarantee. 

Specifically, $ITA$ preserves the semantics of augmented text by adding perturbations to semantic embeddings while constraining its distribution without using extra models. It bypasses the risks of semantic mismatch and enforces the corresponding visual representations of augmented textual embeddings. Crucially, we provide a theoretical basis for $ITA$ enhancing text-image consistency, a premise backed by the group theory for data augmentation~\cite{chen2020group}.
As illustrated in Figure~\ref{fig:training_para}~(b), the augmented text embeddings are engaged with the inference process, providing semantic supervision to enhance their regularization role. On the implementation front, two variants for $ITA$: a closed-form calculation $ITA_C$ (training-free), and its simple learnable equivalent $ITA_T$. It is further proved that a theoretical equivalence of  $ITA_C$ arrives at the same solution to recent methods~\cite{dong2017i2t2i,cheng2020rifegan} that employ auxiliary models for textual augmentation when these auxiliary models are well-trained. This suggests that $ITA_C$ offers an elegant and simplified alternative to prevent semantic mismatch.  

\begin{figure}[t]
\centering
\includegraphics[width=\linewidth]{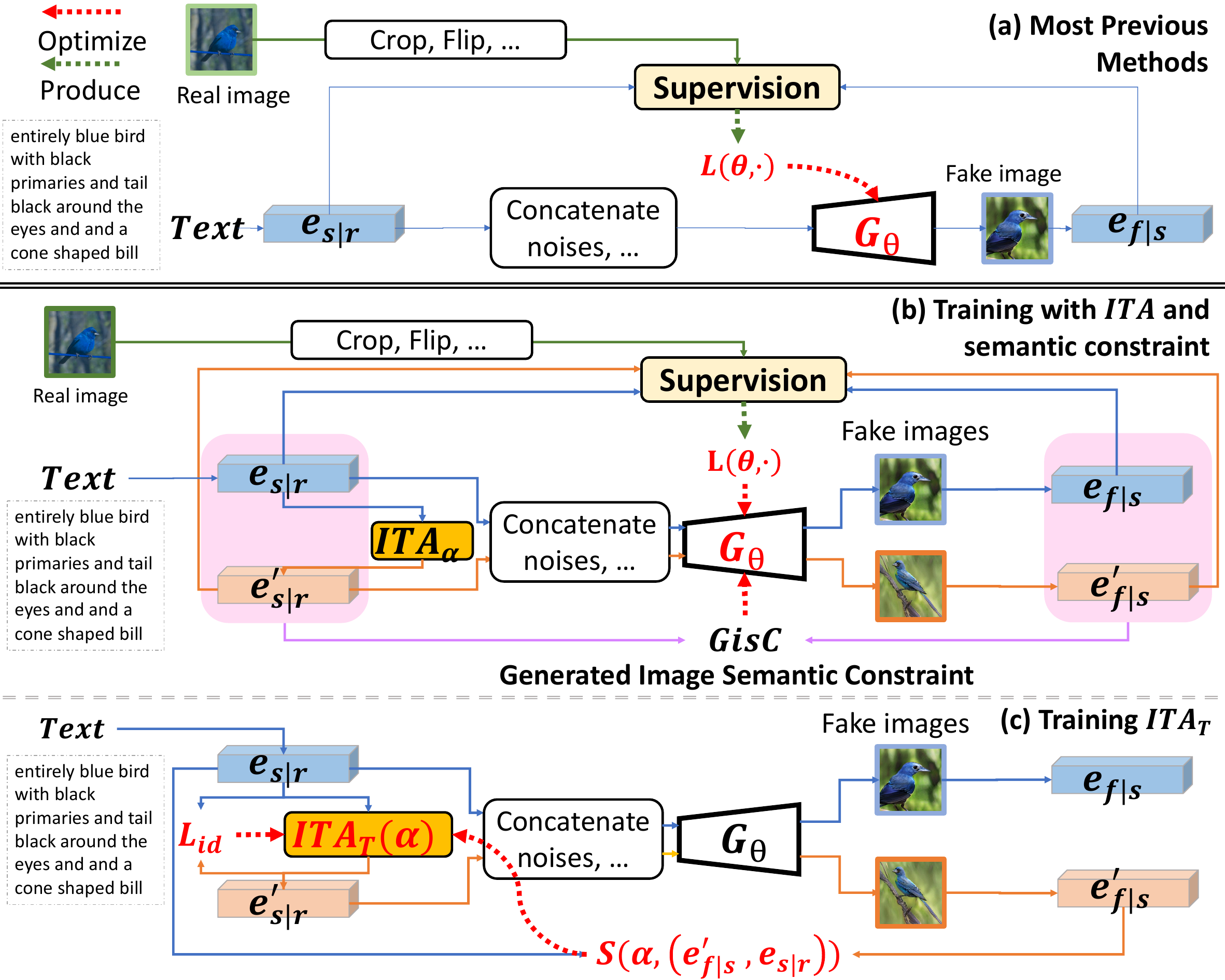}
\caption{$L(\theta,\! \cdot)$ is optimization loss for $G$. $S(\theta,(\cdot,\! \cdot))$ measures semantic consistency. (a) Simplified training paradigm of previous methods. (b) Training paradigm of SADA.
(c) Training of $ITA_T$ where generators are frozen.}
\label{fig:training_para}
\end{figure}
Meanwhile, we  identify that an effective $GisC$ diminishes semantic collapse and benefits the generated image quality.
Inspired by variance-preservation~\cite{bardes2021vicreg}, we design an Image Semantic Regularization Loss ($L_r$) to serve as a $GisC$ with $ITA_C$, which constrains both the semantic shift direction and distance of generated images (see Figure~\ref{fig:r_expalin}~(d)).  Through Lipschitz continuity and semantic constraint tightness analysis (as seen in Propositions~\ref{prop:lip} and \ref{prop:r_better_db}), we theoretically justify that $L_r$ prevents the semantic collapse, consequently yielding superior image quality compared to methods that solely bound semantic direction~\cite{gal2022stylegan}.
\red{Notably, SADA can serve as a theoretical framework for other empirical forms of $ITA$ and $GisC$ in the future.}

Our contributions can be summarized as follows:
\begin{itemize}
    \item This paper proposes a novel Semantic-aware Data Augmentation  (SADA) framework that consists of an Implicit Textual Semantic Preserving Augmentation ($ITA$) and a Generated Image Semantic Conservation ($GisC$).

    \item Drawing upon the group theory for data augmentation~\cite{chen2020group}, we prove that $ITA$ certifies a text-image consistency improvement. As evidenced empirically,  $ITA$ bypasses semantic mismatch while ensuring visual representation for augmented textual embeddings.
    
    \item We make the first attempt to theoretically and empirically show that $GisC$ can additionally affect the raw space to improve image quality.  We theoretically justify that using Image Semantic Regularization Loss $L_{r}$ to achieve $GisC$ prevents semantic collapse through the analysis of Lipschitz continuity and semantic constraint tightness.
    

    \item Extensive experimental results show that SADA can be simply applied to typical T2Isyn frameworks, 
    {such as diffusion-model-based frameworks}, 
    effectively improving text-image consistency and image quality. 
\end{itemize}

\section{Related Work}
\textbf{T2Isyn Frameworks and Encoders:}
Current T2Isyn models have four main typical frameworks: attentional stacked GANs accompanied with a perceptual loss produced by pre-trained encoders~\cite{zhang2017stackgan,zhang2018stackgan++,xu2018attngan,zhu2019dm,ruan2021dae}, one-way output fusion GANs~\cite{tao2022df}, VAE-GANs with transformers~\cite{gu2022vector}, and diffusion models (DMs)~\cite{dhariwal2021diffusion}. 
Two encoders commonly used for T2Isyn are DAMSM~\cite{xu2018attngan,tao2022df} and 
CLIP~\cite{radford2021learning}. 
Our proposed SADA is readily applied to these current frameworks with different encoders.

\noindent\textbf{Augmentations for T2Isyn:} 
Most T2Isyn models~\cite{reed2016generative,xu2018attngan,tao2022df,gu2022vector}
only use basic augmentations such as image corp, flip, and noise concatenation to textual embedding without exploiting further augmentation facilities. To preserve textual semantics, I2T2I~\cite{dong2017i2t2i} and RiFeGAN~\cite{cheng2020rifegan} preserve textual semantics using an extra pre-trained captioning model and an attentional caption-matching model respectively, to generate more captions for real images and to refine retrieved texts for T2Isyn. They still suffer from semantic conflicts between input and retrieved texts, and their costly retrieval process leads to infeasibility on large datasets, prompting us to propose a more tractable augmentation method. 

\noindent\textbf{Variance Preservation:} 
Stylegan-nada~\cite{gal2022stylegan} presents semantic Direction Bounding ($L_{db}$) to constrain semantic shift directions of texts and generated images, which may not guarantee the prevention of semantic collapse. 
Inspired by variance preservation in contrastive learning~\cite{bardes2021vicreg} based on the principle of maximizing the information
content~\cite{ermolov2021whitening,zbontar2021barlow,bardes2021vicreg}, we constrain the variables of the generated image semantic embeddings to have a particular variance along with its semantic shift direction.

\section{Implicit Textual Semantic Preserving Augmentation}


Consider observations $\hat{X}_1, ..., \hat{X}_k \in \hat{\mathcal{X}}$ sampled i.i.d. from a probability distribution $\mathbb{P}$  in the sample space $\hat{\mathcal{X}}$, 
where each $\hat{\mathcal{X}}$ includes real image $r$ and its paired text $s$. According to $\hat{X} \in \hat{\mathcal{X}}$, we then have $X_1, ..., X_k \in \mathcal{X}$ where each $X$ includes real image embedding $e_r$ and text embedding $e_s$. 
We take $G$ with parameter $\theta$ as a universal annotation for generators in different frameworks; $L(\theta, \cdot)$ represents total losses for $G$ used in the framework. 
Following the Group-Theoretic Framework for Data Augmentation~\cite{chen2020group}, we also assume that:

\begin{assumption}
\label{ass:invariance}
    If original and augmented data are a group that is exact invariant (i.e., the distribution of the augmented data is equal to that of the original data),  semantic distributions of texts/images are exact invariant.
\end{assumption}



Consider augmented samples $X' \in \mathcal{X}'$, 
where $X'$ includes $e_r$, and augmented textual embedding $e_s'$.
According to Assumption~\ref{ass:invariance}, we have an equality in distribution:
\begin{align}
    \mathcal{X} =_d \mathcal{X}',
\end{align}
which infers that both $X$ and $X'$ are sampled from $\mathcal{X}$.
Bringing it down to textual embedding specifically, we further draw an assumption:
\begin{assumption}
\label{ass:aug}
    If the semantic embedding $e_s$ of a given text follows a distribution $Q_{s}$, then $e_s'$ sampled from $Q_{s}$ also preserves the main semantics of $e_s$. 
\end{assumption}

This assumption can be intuitively understood to mean that for the given text, there are usually a group of synonymous texts. 
Satisfying exact invariant, $e_s'$ sampled from $Q_{s}$ preserves the main semantics of $e_s$.  $e_s'$ can be guaranteed to drop within the textual semantic distribution and correspond to a visual representation that shares the same semantic distribution with the generated image on $e_s$. Thus, $e_s'$ can be used to generate a reasonable image.
Under Assumption~\ref{ass:aug},  we propose the Implicit Textual Semantic Preserving Augmentation ($ITA$) that can obtain $Q_s$.
As shown in Figure~\ref{fig:r_expalin}~(a)(b), $ITA$ boosts the generalization of the model by augmenting implicit textual data under $Q_s$.

\subsection{Training Objectives for $G$ with $ITA$}

The general sample objective with $ITA$ is defined as:
\begin{align}
\label{eq:itatsolution}
    \min_\theta \Hat{R}_k(\theta):= \frac{1}{k} \sum\nolimits_{i=1}^k  {L}(\theta, ITA(X_i)).
\end{align}
We then define the solution of $\theta$ based on Empirical Risk Minimization (ERM)~\cite{ERM} 
as:
\begin{align}
    \text{ERM: } \theta_{ITA}^* \in arg \min_{\theta \in \Theta }\frac{1}{k} \sum\nolimits_{i=1}^k  {L}(\theta, ITA(X_i)) ,  
\end{align}
where $\Theta$ is defined as some parameter space. See detailed derivation based on ERM in Supplementary Materials~\ref{app:ita_obj}.

\begin{proposition}[\textbf{$ITA$ increases T2Isyn semantic consistency}] 
\label{prop:semantic_improvement}
    Assume exact invariance holds.
    Consider an unaugmented text-image generator  $\Hat{\theta}(X)$ of $G$ and its augmented version $\Hat{\theta}_{ITA}$. For any real-valued convex loss $S(\theta, \cdot)$ that measures the semantic consistency, we have:
    \begin{align}
    \label{eq:consisteny}
        \mathbb{E} [S(\theta, \Hat{\theta}(X))] \ge \mathbb{E} [S(\theta, \Hat{\theta}_{ITA} (X))] ,
    \end{align}
    which means with $ITA$, a model can have lower $\mathbb{E} [S(\theta, \Hat{\theta}_{ITA} (X)]$ thus a better text-image consistency. 
\end{proposition}

\begin{proof} 
   we obtain a direct consequence that: $Cov [ \Hat{\theta}_{ITA}(X)]$ $\preceq$ $Cov [\Hat{\theta}(X)]\;,$
   where $Cov[\cdot]$ means the covariance matrix decreases in the Loewner order.
   Therefore, $G$ with $ITA$ can obtain better text-image consistency. See proof details in Supplementary Materials~\ref{app:proof}.
\end{proof}
For a clear explanation, we specify a form $S(\theta,\cdot ):= S(\theta,(\cdot, \cdot))$ where $(\cdot, \cdot)$ take a $e_s$ and $e_r$ for semantic consistency measuring, and $\theta$ denotes the set of training parameters. 
Since we preserve the semantics of $e_s'$, its generated images should also semantically match $e_s$. Thus,
the total semantic loss of $G$ is defined as: 
\begin{align}
\label{eq:Ls} 
   { L_{S} =} & {S(\theta, (e_s, \mathcal{G}(e_s)) ) + S(\theta, (e_s', \mathcal{G}(e_s')) ) } \notag \\
    & {+ S(\theta, (e_s, \mathcal{G}(e_s')) ) + S(\theta, (e_s', \mathcal{G}(e_s)) ) \; },
\end{align}
where $\mathcal{G} = h(G(\cdot))$, $(\cdot)$ takes a textual embedding and $h(\cdot)$ maps images into semantic space. Typically, as the first term is included in the basic framework, it is omitted while other terms are added  for SADA applications.


\begin{figure}[t]
\begin{center}
\centerline{
 \includegraphics[width=\linewidth]{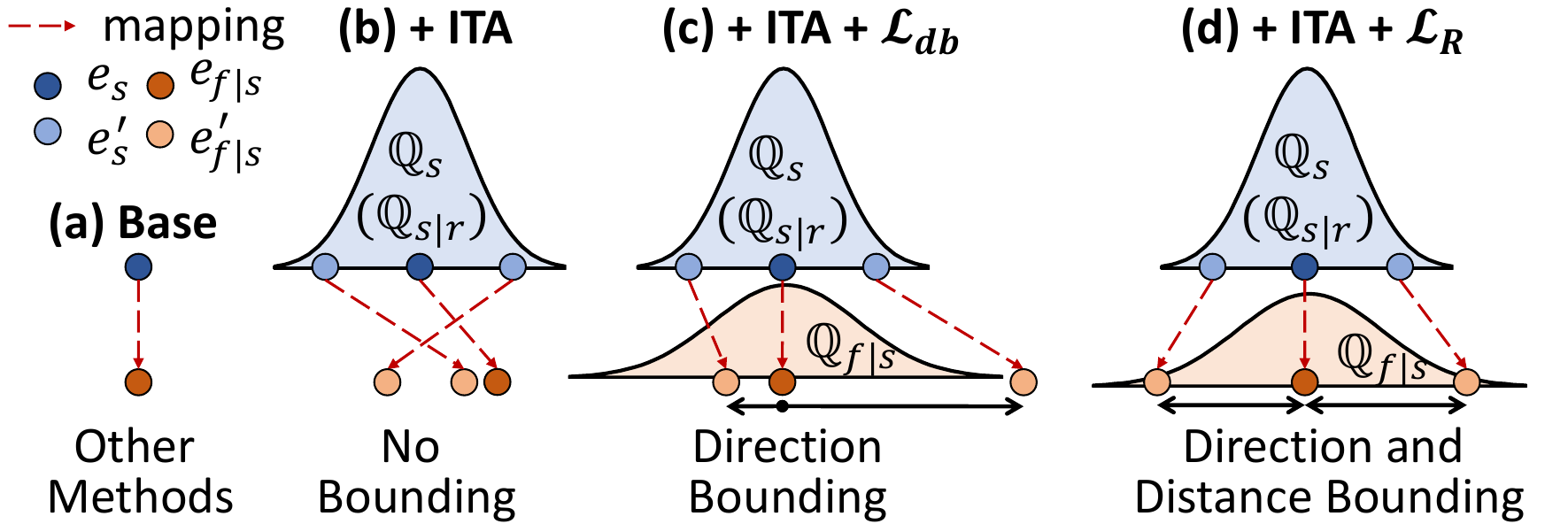}
}
\caption{Diagram of augmentation effects of our proposed SADA ($+ITA, +ITA+L_{db}$, $ITA+L_r$).
}
\label{fig:r_expalin}
\end{center}
\vskip -0.1in
\end{figure}

\subsection{Obtaining Closed-from $ITA_C$} 
\subsubsection{Theoretical Derivation of $ITA_C$}
Assume that exact invariance holds. We treat each textual semantic embedding $e_s$ 
as a Gaussian-like distribution $\phi = \mathcal{N}(e_s, \sigma)$, where each sample $e_s' \sim \mathcal{N}(e_s, \sigma)$ can maintain the main semantic $m_s$ of $e_s$. 
In other words,  $\sigma$ is the variation range of $e_s$ conditioned by $m_s$, $\phi$ derives into:
\begin{align}
\label{eq:phi1}
    \phi = \mathcal{N}(e_s, \sigma|m_s) \; .
\end{align}

By sampling $e_s'$ from $\phi$, we can efficiently obtain augmented textual embedding for training. 
We need  to draw support from real images to determine the semantics $m_s$ that need to be preserved. Empirically, real texts are created based on real images.
$e_s$ is thus naturally depending on $e_r$, leading to the inference: $e_{s|r} \triangleq  e_{s}, m_{s|r} \triangleq  m_s, Q_{s|r} \triangleq  Q_{s}$.
Given a bunch of real images, $\sigma|m_s$ is assumed to represent the level of variation inherent in text embeddings, conditioned on the real images.  
We can redefine $\phi$ in Eq.~(\ref{eq:phi1}) for $ITA_C$ augmentation as: $\phi  \triangleq  \mathcal{N}(e_{s|r},  \sigma|m_{s|r}) =  \mathcal{N}(e_{s|r}, \beta \cdot \mathbb{C}_{ss|r} \mathbb{I}),$
where $\mathbb{C}_{**}$ denotes covariance matrix of semantic embeddings; 
$r, s$ stand for real images and real texts; $\mathbb{C}_{ss|r}$ is the self-covariance of $e_s$ conditioned by semantic embedding of real images $e_r$; $\mathbb{I}$ denotes an identity matrix; $\beta$ is a positive hyper-parameter for controlling sampling range. As such, we define: $\phi \triangleq {Q}_{s|r}$. 
According to~\cite{kay1993fundamentals}, conditional $\mathbb{C}_{ss|r}$ is equivalent to:
\begin{align}
\label{eq:cond_ss_r}
\mathbb{C}_{ss|r} = \mathbb{C}_{ss} - \mathbb{C}_{sr}\mathbb{C}_{rr}^{-1}\mathbb{C}_{rs}\;,
\end{align}
where all covariances can be directly calculated.
Then $\phi$ is calculated from the dataset using semantic embeddings of texts and images for $s$ and $r$. 
In practice, $\mathbb{C}_{ss|r}$ is calculated using real images and their given texts from the training set.


\subsubsection{Remarks of $ITA_C$}
We explore the connections between $ITA_C$ and previous methods~\cite{dong2017i2t2i,cheng2020rifegan}, assuming all models are well-trained. 
\begin{proposition}
\label{prop:close_form}
    $ITA_C$ can be considered a closed-form solution for general textual semantic preserving augmentation methods of T2Isyn. 
\end{proposition}
Proof details can be seen in Supplementary Materials~\ref{app:proof}.
Therefore, training with bare $ITA_C$ is equivalent to using other textual semantic preserving augmentation methods.

\subsubsection{$ITA_C$ Structure}
Based on Eq.~(\ref{eq:cond_ss_r}), we obtain $e_{s|r}'$ from calculated $ITA_C$:
\begin{align}
\label{eq:esaug}
    e_{s|r}' = e_{s|r}' \sim \phi = e_{s|r} + z \triangleq e_{s|r} + \epsilon \odot \beta \cdot  \mathbb{C}_{ss|r} \mathbb{I},
\end{align}
where $z \sim \mathcal{N}(0, \beta \cdot \mathbb{C}_{ss|r} \mathbb{I})$, $\epsilon$ is sampled from a uniform distribution $U(-1,1)$, as shown in Figure~\ref{fig:ITA}.  $ITA_C$ requires no training and can be used to train or tune a T2Isyn model.

\begin{figure}[t]
\begin{center}
\centerline{
\includegraphics[width=\linewidth]{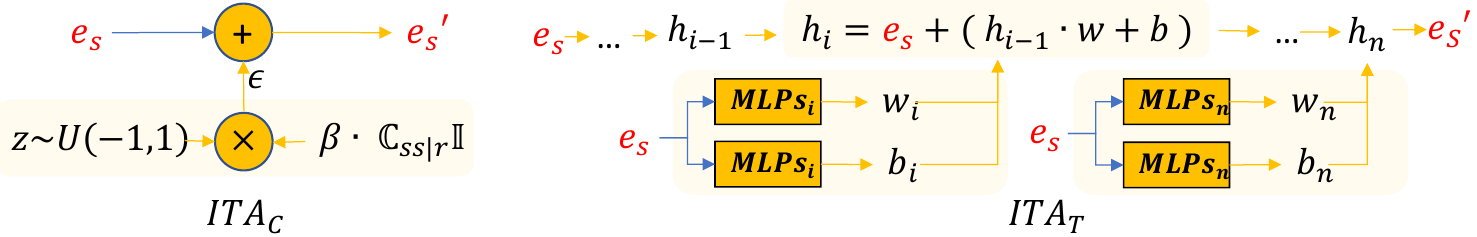}
}
\caption{Network structure of $ITA_C$ and $ITA_T$. \red{Note that $e_{s}$ and $e'_{s}$ are equivalent to $e_{s|r}$ and $e'_{s|r}$ respectively. }}
\label{fig:ITA}
\end{center}
\vskip -0.15in
\end{figure}

\subsection{Obtaining Learnable $ITA_T$}

We also design a learnable $ITA_T$ as a clever substitute. Proposition~\ref{prop:close_form} certifies that well-trained $ITA_T$ is equivalent to $ITA_C$.
To obtain $ITA_T$ through training, we need to  achieve the following objectives:
\begin{align}
   \max_\alpha {L}_{d}( \alpha,(e_{s|r}', e_{s|r})) ,\;\min_\alpha  S(\alpha, (e_{s|r}, \mathcal{G}(e_{s|r}')) ) \notag\; ,
\end{align}
where 
${L}_{d}( \alpha, \cdot, \cdot)$ denotes a distance measurement, 
enforcing that the augmented $e_{s|r}'$ should be far from $e_{s|r}$ as much as possible; $\alpha$ is training parameters of ${ITA_T}$. $S(\alpha,(\cdot, \cdot))$  bounds the consistency between $e_{s|r}$ and generated images on $e_{s|r}'$ ,
preserving the semantics of $e_{s|r}'$. 
The first objective can be easily reformed as minimizing the inverse distance:
\begin{align}
     &\min_\alpha {L}_{id}( \alpha,(e_{s|r}', e_{s|r})) :=  \min_\alpha - {L}_{d}( \alpha,(e_{s|r}', e_{s|r})) . \notag
\end{align}
The final loss for training $ITA_T$ is a weighted combination of ${L}_{d}$ and $S(\alpha,(\cdot, \cdot))$:
\begin{align} 
\label{eq:Litat}
      {L}_{ITA_T} = & r \cdot {L}_{id}( \alpha,(e_{s|r}', e_{s|r})) \nonumber \\ &+ ( 1-r) \cdot S(\alpha, (e_{s|r}, \mathcal{G}(e_{s|r}')) ,
\end{align}
where $r$ is a hyper-parameter controlling the augmentation strength. Note that ${L}_{ITA_T}$ is only used for optimizing $\alpha$ of $ITA_T$ and parameters of $G$ are frozen here (as Figure~\ref{fig:training_para}~(c)).


\subsubsection{$ITA_T$ Structure}
Since the augmented $e_{s|r}'$ should maintain the semantics in $e_{s|r}$, $\epsilon$ in Eq.~(\ref{eq:esaug}) is maximized but does not disrupt the semantics in $e_{s|r}$. 
As such, $\epsilon$ is not a pure noise but a $e_{s|r}$-conditioned variable. Hence, Eq.~(\ref{eq:esaug}) can be reformed as $ e_{s|r}' = e_{s|r} + f(e_{s|r})$ to achieve $ITA_T$,
where $f(e_{s|r})$ means a series of transformations of $e_{s|r}$. The final $ITA_T$ process can be formulated as $e_{s|r}' = ITA_T(e_{s|r}) = e_{s|r} + f(e_{s|r})$.
We deploy a recurrent-like structure as shown in Figure~\ref{fig:ITA} to learn the augmentation. 
$ITA_T$ takes $e_{s|r}$ as an input. For $i^{th}$ step in overall $n$ steps, there is a group of Multilayer Perceptrons to learn the weights $w_i$ and bias $b_i$ conditioned by $e_{s|r}$ for the previous module's output $h_{i-1}$. Then $h_i = e_{s|r} +(h_{i-1} \cdot w_i + b_i )$ will be output to the following processes.
We empirically set $n=2$ for all our experiments.
$ITA_T$ can be trained simultaneously with generative frameworks from scratch or used as a tuning trick. 


\section{ Generated Image Semantic Conservation }
Enabled by $ITA$'s providing $e_{s|r}, e_{s|r}' $,
we show that using Generated Image Semantic Conservation  ($GisC$) will affect generated images' raw space. 
Consider a frozen pre-trained image encoder ($E_I$) that maps images into the same semantic space. Consider a feasible and trainable generator $G$ that learns how to generate text-consistent images:
$
    G(X) \to  \mathcal{F}, \; E_I(\mathcal{
    F}) \to \mathcal{E} 
$,
where $\mathcal{F}$ and $\mathcal{E}$ are the sets for generated images $f$ and their semantic embeddings $e_f$. Since images are generated on texts, we have $e_{f|s} \triangleq e_f$.
We show that semantically constraining generated images can additionally affect their raw space.

\begin{proposition}
\label{prop:enhance_image}
Assume that $E_I$ is linear and well-trained. Constraining the distribution $Q_\mathcal{E}$ of $e_{f|s}$ can additionally constrain the distribution $\mathcal{F}$ of $f$. 
\end{proposition} 
\begin{proof}
    There are two scenarios: 
     1) If $E_I$ is inevitable, Proposition~\ref{prop:enhance_image} is obvious.
     2) If $E_I$ is not inevitable, it is impossible that  $\mathcal{F}$ all locates in the $Null(E_I)$ (nullspace of $E_I$) for well trained $E_I$, thus constraining $\mathcal{F}$ can affect $\mathcal{E}$.
     See more proof details in Supplementary Materials~\ref{app:proof}.
\end{proof}
We further assume that the positive effeteness of feasible $GisC$ can pass to the raw generated image space. 
The non-linear case is non-trivial to proof. 
Our results of using non-linear encoders (DAMSM~\cite{xu2018attngan} and CLIP~\cite{radford2021learning}) with different feasible $GisC$ methods suggest that Proposition~\ref{prop:enhance_image} holds for non-linear $E_I$ and positively affect image quality.




\subsection{Image Semantic Regularization Loss }


We design an Image Semantic Regularization Loss $L_r$ to attain $GisC$ for preventing semantic collapse and providing tighter semantic constraints than direction bounding $\mathcal{L}_{db}$~\cite{gal2022stylegan}. 

\subsubsection{Theoretical Derivation of $L_r$}
To tackle semantic collapse empirically, we constrain the semantic distribution of generated images, which draws inspiration from the principle of maximizing the information content of the embeddings through variance preservation~\cite{bardes2021vicreg}. 
Since semantic redundancies undescribed by texts in real images are not compulsory to appear in generated images, the generated images are not required to be the same as real images.
Therefore, conditioned by the texts, generated images should obtain semantic variation in real images.  
For example, when text changes from `orange' to `banana', `orange' in real images should likewise shift to `banana' despite the redundancies, and fake images should obtain this variance~\cite{tan2023semantic}. 
If exact invariance holds and the model is well-trained,
the text-conditioned semantic distribution of its generated images $Q_{f|s} = \mathcal{N}(m_{f|s}, \mathbb{C}_{ff|s}\mathbb{I})$ should have the semantic variance as close as that of the real images $Q_{rr|s} = \mathcal{N}(m_{r|s}, \mathbb{C}_{rr|s}\mathbb{I})$:
    \begin{align}
    \label{eq:ffs_rrs}
        \min_{e_f} ||\mathbb{C}_{ff|s}\mathbb{I}\! - \! \mathbb{C}_{rr|s}\mathbb{I} ||^2, \mathbb{C}_{rr|s}\!  = \! \mathbb{C}_{rr}\!  - \! \mathbb{C}_{rs}\mathbb{C}_{ss}^{-1}\mathbb{C}_{sr} \;,
    \end{align}
where $\mathbb{C}_{rr|s}$ is the self-covariance of $e_r$ conditioned by real text embeddings.

Aim to maintain latent space alignment, an existing $GisC$ method, direction bonding~\cite{gal2022stylegan} is defined as:
\begin{align}
\label{eq:Lb}
  {L}_{db} = 1 - \frac{( e_{s|r}'-e_{s|r} ) \cdot (e_{f|s}'- e_{f|s}) }{||( e_{s|r}'-e_{s|r} )||^2 \cdot ||(e_{f|s}'- e_{f|s}) ||^2} \;.
\end{align}
$L_{db}$ follows that semantic features are usually linearized~\cite{bengio2013better,upchurch2017deep,wang2021regularizing}. 
Given a pair of encoders that maps texts and images into the same semantic space, inspired by $L_{db}$, we assume that:
\begin{assumption}
\label{ass:shifts}
    If the paired encoders are well-trained, aligned, and their semantic features are linearized. The semantic shifts images are proportional to texts:
    \begin{align}
    \label{eq:shift}
        ( e_{f|s}' - e_{f|s}  ) \propto ( e_{s|r}' - e_{s|r} ).
    \end{align}
\end{assumption}

Assumption~\ref{ass:shifts}
holds for T2Isyn intuitively because when given textual semantics changes, its generated image's semantics also change, whose shifting direction and distance are based on   textual semantics changes. Otherwise, semantic mismatch and collapse would happen.
If Assumption~\ref{ass:shifts} holds, 
based on $ITA_C$ that preserves $e_{s|r}' - e_{s|r}$,  we have:
\begin{align}
\label{eq:Lr_requires}
     e_{f|s}' - e_{f|s}  \le \epsilon \odot \beta \cdot d(         \mathbb{C}_{ff|s})  \notag
     \\ \text{s.t. } e_{s|r}' - e_{s|r} \le \epsilon \odot \beta \cdot  d( \mathbb{C}_{ss|r})\;.
\end{align}
If we force that each dimension of ${\epsilon^*}^{d}_{i =1} \sim \{-1, 1\}$ where $d = \{1,...,n\}$ and $n$ is the dimension of the semantic embedding, we have:
\begin{align}
\label{eq:Lr_requires_final}
     e_{f|s}'' - e_{f|s}  = {\epsilon^*} \odot \beta \cdot d( \mathbb{C}_{ff|s})  \notag
     \\ \text{s.t. } e_{s|r}'' - e_{s|r} = {\epsilon^*} \odot \beta \cdot  d( \mathbb{C}_{ss|r}) \; .
\end{align}
Derived form Eqs.~(\ref{eq:ffs_rrs})~and~(\ref{eq:Lr_requires_final}), 
we define our Image Semantic Regularization Loss $L_r$ as:
\begin{align}
\label{eq:R}
     L_r = \varphi \cdot || \; ( e_{f|s}'' - e_{f|s}) - {\epsilon^*} \odot \beta \cdot d(\mathbb{C}_{rr|s}) || ^2 \;,
\end{align}
where $\beta \cdot d(\mathbb{C}_{ff|s})$ can be considered a data-based regularized term.
$\epsilon$ constrains the shifting direction, as shown in Figure~\ref{fig:r_expalin} (d). $\varphi$ is a hyper-parameter for balancing $L_r$ with other loss.
Note that for $ITA_T$, the range of $e_{s|r}' - e_{s|r}$ is not closed-form. Thus we cannot apply $L_r$ with $ITA_T$. 

\subsubsection{Remarks of $L_r$}

We show the effect of $L_r$ on the semantic space of generated images:
\begin{proposition}[\textbf{$L_r$ prevent semantic collapse: completely different}]
\label{prop:lip}
$L_r$ leads to $| e_{f|s}'-e_{f|s} |$ is less than or equal to a sequence $\Lambda$ of positive constants, further constrains the semantic manifold of generated embeddings to meet the Lipschitz condition.
\end{proposition}

\begin{proof}
From Eq.~(\ref{eq:R}),  we have the constraint $| e_{f|s}' - e_{f|s} | \leq \Lambda$. Therefore, we have: $\frac{| e_{f|s}' - e_{f|s} |}{ | e_{s|r}' - e_{s|r} |} \leq K, \; \text{s.t.} \; e_{s|r}' \neq e_{s|r} \; ,$
where $K$ is a Lipschitz constant.
See more proof details in Supplementary Materials~\ref{app:proof}.
\end{proof}


Proposition~\ref{prop:lip} justifies why image quality can be improved with $L_r$.  
According to Proposition~\ref{prop:enhance_image}, we believe that the Lipschitz continuity can be passed to visual feature distribution, leading to better  continuity in visual space as well.  Our experiments verify that with $L_r$ methods, T2Isyn models achieve the best image quality.

\begin{proposition}[\textbf{$L_r$ prevent semantic collapse: extremely similar}]
\label{prop:r_better_db}
 $L_r$ prevents $| e_{f|s}''-e_{f|s} | = 0$ and provides tighter image semantic constraints than direction bounding ${L}_{db}$.
\end{proposition}

\begin{proof}
For Eq.~(\ref{eq:Lb}),
assume $ {L}_{db} = 0$ and use $e_{s|r}''$ to substitute $e_{s|r}$, combining with Eq.~(\ref{eq:esaug}),
we have: $ | e_{f|s}''-e_{f|s} | \ge 0 \;.$
Preservation of semantic collapse is not guaranteed due to the distance between $e_{f|s}''\;(e_{f|s}') $ and $e_{f|s}$ is not strictly contained. 
Assume $L_r = 0$, we have: $| e_{f|s}'' - e_{f|s} | > 0 \; ,$
where provides tighter constraints than $L_{db}$. See visual explanation in Figure~\ref{fig:r_expalin}~(c)(d) and proof details in Supplementary Materials~\ref{app:proof}.
\end{proof}
Propositions~\ref{prop:lip}-\ref{prop:r_better_db} 
show that $L_r$ prevents semantic collapse. See SADA' algorithms in Supplementary Materials~\ref{app:algorithms}.

\section{Experiments}
\label{sec:exp}
Our experiments include three parts:
1) To demonstrate how $ITA$ improves text-image consistency, we apply $ITA$ of SADA to Text-Image Retrieval tasks. 
2) To exhibit the feasibility of our SADA, we conduct extensive experiments by using different T2Isyn frameworks with GANs, Transformers, and Diffusion Models (DM) as backbones on different datasets. 
3) Detailed ablation studies are performed; we compare our SADA with other typical augmentation methods to show that SADA certifies an improvement in text-image consistency and image quality in T2Isyn tasks. Particularly noteworthy is the observation that $GisC$ can alleviate semantic collapse. Due to page limitations, key findings are presented in the main paper. For detailed application and training information,
as well as more comprehensive results and visualizations, please refer to Supplementary Materials~\ref{app:More_Experimental_Details} and~\ref{app:More_Results}. \red{Codes are available at \url{https://github.com/zhaorui-tan/SADA}.}


\begin{figure*}[t]
    \centering
    \includegraphics[width=\linewidth]{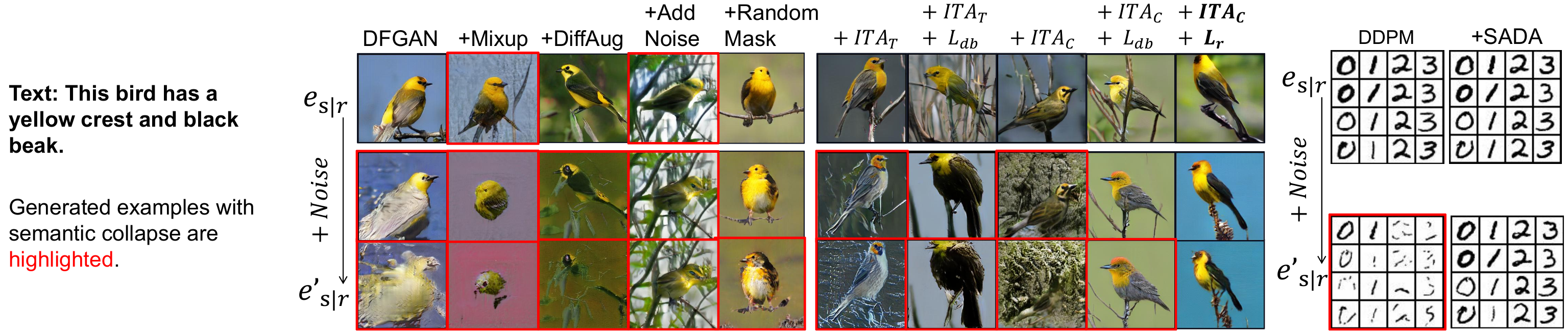}
    \caption{Generated examples of DF-GAN and DDPM  trained with different augmentations on $e_{s|r}$ as ascending $Noise\! \sim \!\mathcal{N}(0, \!\beta \cdot \mathbb{C}_{ss|r} \mathbb{I})$ is given.  Input noise is fixed for each column. See full examples in Supplementary Materials Figures~\ref{fig:DFCUB},~\ref{fig:DFCOCO}~\&~\ref{fig:diffaug}.}
    \label{fig:ablation}
\end{figure*}

\subsection{SADA on Text-Image Retrieval}
\subsubsection{Experimental setup} We compare tuning CLIP~\cite{wang2022clip}(ViT-B/16) performance w/ $ITA$ and wo/ $ITA$ on the COCO~\cite{lin2014microsoft} dataset. 
Evaluation is based on Top1 and Top5 retrieval accuracy under identical hyperparameter settings.
\subsubsection{Results} As exhibited in Table~\ref{tab:retrieval_res}, using $ITA$ results in a boost in image-text retrieval accuracy in both the Top1 and Top5 rankings, reflecting its proficiency in enhancing the consistency between text and images. The increase of $0.45\%$ and $1.56\%$ in Top1 retrieval accuracy explicitly suggests a precise semantic consistency achieved with SADA, providing empirical validation to our Proposition~\ref{prop:semantic_improvement}. 

\begin{table}[t]
\small
    \centering
    \begin{tabular}{p{0.7cm}|p{1.3cm}p{1.45cm}|p{1.3cm}p{1.45cm}}
    \hline
        & \multicolumn{2}{c|}{Image Retrieval} &\multicolumn{2}{c}{Text Retrieval} \\ \hline
         & Top1 & Top5  & Top1 & Top5  \\ 
         \hline
        CLIP 
        & 30.40 & 54.73 &  49.88 & 74.96  \\ \hline
        Tuned & 44.43 & 72.38 & 61.20 & 85.16 \\
        \textbf{+$ITA$}
        & \textbf{44.88}{(+0.45)} & \textbf{72.42}{(+0.04)} & \textbf{62.76}{(+1.56)} &  \textbf{85.38}{(+0.22)} \\
        \hline
    \end{tabular}
    %
    \caption{Text-Image Retrieval results of CLIP tune w/ and wo/ SADA. \red{Please refer to Supplementary Material~\ref{app:More_Results_clip}for tuning CLIP with different number of samples.}} 
    \label{tab:retrieval_res}
\end{table}
\begin{table}[t]
    \centering
    \small
    \begin{tabular}{p{1.6cm}|cc|cccc}
    \hline
        Backbone &  \multicolumn{2}{c|}{\begin{tabular}[c]{@{}l@{}}Encoder,  Method\\ Settings, Dataset\end{tabular}} &  CS$\uparrow$ & FID$\downarrow$  \\\hline
         \textbf{Transformer} & CLIP & VQ-GAN+CLIP & 62.78 & {16.16}\\
         {+SADA} & Tune & COCO &  \textbf{62.81} & {\textbf{15.56}} \\ \hline
         \textbf{DM} & CLIP &  SD &  72.72 & 55.98  \\
        {+SADA}& Tune & Pok\'emon BLIP & \textbf{73.80} & \textbf{46.07}
        \\
        \hline 
        \textbf{DM} & CLIP &  DDPM &  70.77 &  8.61 \\ 
         {+SADA} & Train &  MNIST &   \textbf{70.91} & \textbf{7.78} \\ \hline
        \textbf{GANs} &  DAMSM &  AttnGAN & 68.00 & 23.98  \\
        {+SADA}& Train &  CUB &   \textbf{68.20} & \textbf{13.17}
        \\
        \hline 
         \textbf{GANs} &  DAMSM &  AttnGAN & 62.59 & 29.60  \\
         {+SADA}& Tune &  COCO &  \textbf{64.59} & \textbf{22.70}
        \\
        \hline 
        \textbf{GANs} & DAMSM &  DF-GAN &58.10 & 12.10  \\
         {+SADA} & Train  &   CUB &\textbf{58.24} & \textbf{10.45}
        \\
        \hline 
         \textbf{GANs} & DAMSM &  DF-GAN & 50.71  &15.22   \\
         {+SADA} & Train & COCO & \textbf{51.02} & \textbf{12.49}
        \\
        \hline
    \end{tabular}
    %
    \caption{Performance evaluation of SADA with different backbones with different datasets. Results better than the baseline are in \textbf{bold}. } 
    \label{tab:all_res}
\end{table}
\subsection{SADA on Various T2Isyn Frameworks}
\subsubsection{Experimental setup}
We test SADA on GAN-based AttnGAN~\cite{xu2018attngan} and DF-GAN~\cite{tao2022df}, transformer-based VQ-GAN+CLIP~\cite{wang2022clip}, vanilla DM-based  conditional DDPM~\cite{ho2020denoising} and Stable Diffusion (SD)~\cite{rombach2021highresolution} with different pretrianed text-image encoders (CLIP and DAMSM~\cite{xu2018attngan}).
Parameter settings follow the original models of each framework for all experiments unless specified.
Datasets CUB~\cite{wah2011caltech}, COCO~\cite{lin2014microsoft}, MNIST, and Pok\'emon BLIP~\cite{deng2012mnist} are employed for training and tuning (see the $2^{nd}$ column in Table~\ref{tab:all_res} for settings). Supplementary Material~\ref{app:More_Results_SD_more} offers additional SD-tuned results. 
For qualitative evaluation, we use CLIPScore (CS)~\cite{hessel2021clipscore} to assess text-image consistency (scaled by $100$) and Fr{\'e}chet Inception Distance (FID)~\cite{heusel2017gans} to evaluate image quality (computed over 30K generated images).

\subsubsection{Results} 
As shown in Table~\ref{tab:all_res} and corresponding Figure~\ref{fig:all_vis}, 
the effectiveness of our SADA can be well supported by improvements across all different backbones, datasets, and text-image encoders, which experimentally validate the efficacy of SADA in enhancing text-image consistency and image quality. 
Notably, facilitated by $ITA_C\!+\!L_r$, AttnGAN achieves $13.17$ from $23.98$ on CUB. 
For tuning  VQ-GAN+CLIP and SD that have been pre-trained on large-scale data, SADA still guarantees improvements. These results support Propositions~\ref{prop:semantic_improvement},~\ref{prop:enhance_image} and~\ref{prop:lip}. 
It's worth noting that the tuning results of models with DM backbones (SD) are influenced by the limited size of the Pok\'emon BLIP dataset, resulting in a relatively high FID score.  Under these constraints, tuning with SADA performed better than the baseline, improving the CS from $72.72$ to $73.80$ and lowering the FID from $55.98$ to $46.07$.


\begin{figure}
    \centering
    \includegraphics[width=\linewidth]{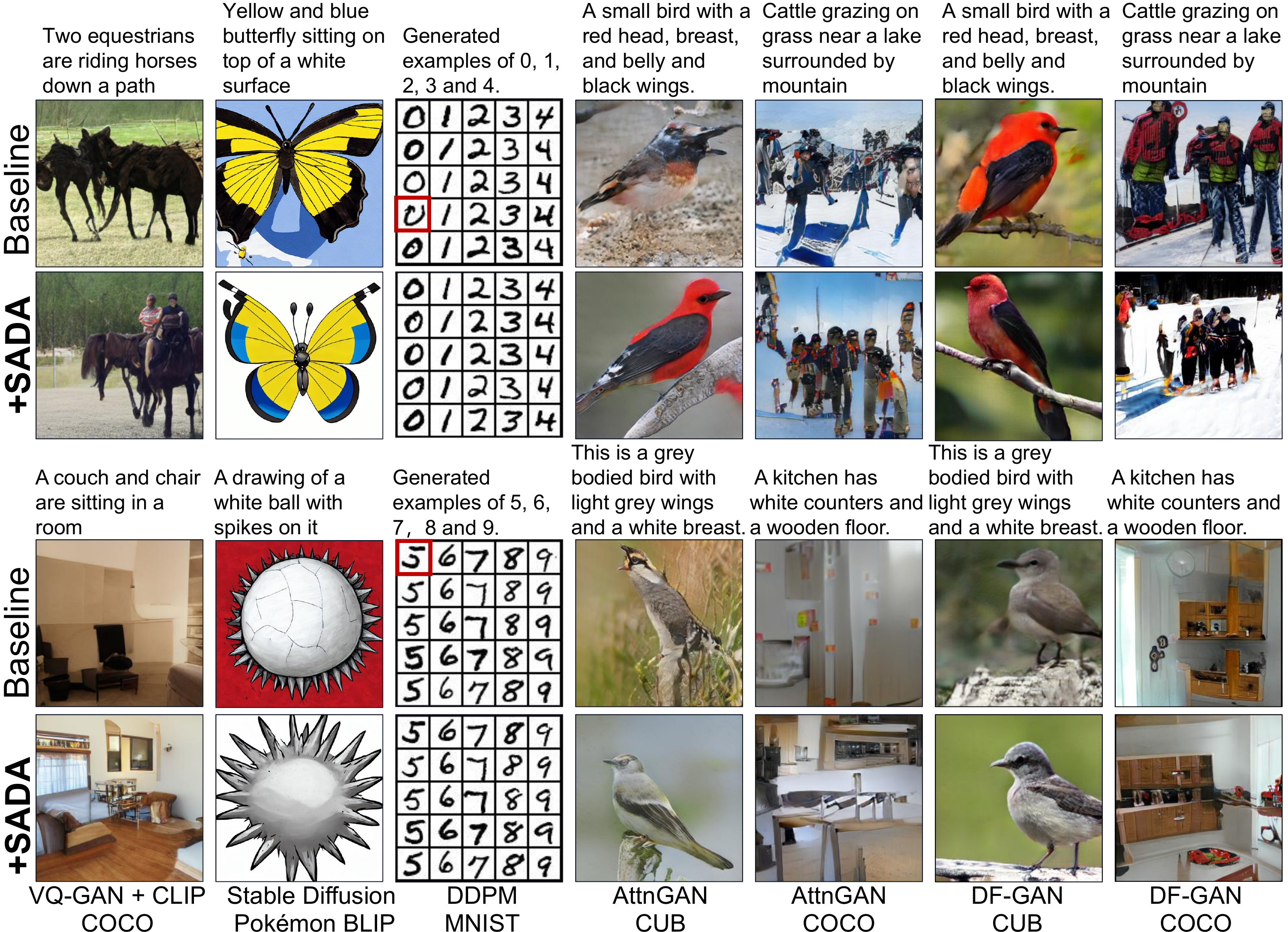}
    \caption{Generated examples of different backbones with different datasets wo/ SADA and w/ SADA. See more examples of different frameworks in Supplementary Materials~\ref{app:More_Results}.}
    \label{fig:all_vis}
\end{figure}

\subsection{Ablation Studies} 
\subsubsection{Experimental setup} 
Based on AttnGAN and DF-GAN, we compare Mixup~\cite{zhang2017mixup}, DiffAug~\cite{zhao2020differentiable}, Random Mask (RandMask), Add Noise, with SADA components in terms of CS and FID. Refer to Supplementary Materials~\ref{app:More_Experimental_Details},~\ref{app:More_Results_ITA_T_r} for more detailed settings and the impact of $r$ in $ITA_T$. 
\subsubsection{Quantitative results} Quantitative results are reported in Table~\ref{tab:results}.\footnote{Note for task 2, we use the best results among current augmentations as the baseline since no released checkpoint is available.}  We discuss the results from different aspects. 

{1). Effect of other competitors:} Mixup and DiffAug weaken visual supervision, resulting in worse FID than baselines. They also waken text-image consistency under most situations.
Moreover, 
Random Mask and Add Noise are sensitive to frameworks and datasets, thus they  
cannot guarantee consistent improvements.

{2). $ITA$ improves text-image consistency:}
Regarding text-image consistency, using $ITA$ wo/, or w/ $GisC$
all lead to improvement in semantics, supporting Proposition~\ref{prop:semantic_improvement}.  
However, $ITA_T$ consumes more time to converge due to its training, weakening its semantic enhancement at the early stage (as in Task~5). 
As it converged with longer training time, $ITA_T$ improves text-image consistency as in Task~6.

{3). $GisC$ promotes image quality:} 
For image quality, it can be observed that using  bare $ITA$ wo/ $GisC$, FID is improved in most situations;
but using constraints such as $L_{db}$ and $L_r$ with $ITA_T$ and $ITA_C$ can further improve image quality except $ITA_T + L_{db}$ in Task~1. These support our Proposition~\ref{prop:enhance_image} and Proposition~\ref{prop:lip}.

{4). $L_r$ provides a tighter generated images semantic constraint than $L_{db}$:} 
Specifically,
compared with $L_{db}$, using our proposed $L_r$ with $ITA_C$ provides the best FID and is usually accompanied by a good text-image consistency, thus validating  our Proposition~\ref{prop:r_better_db}.

\begin{table}[!t]
\begin{center}
\small
\begin{tabular}{p{1.58cm}|p{0.55cm}p{0.55cm}|p{0.55cm}p{0.55cm}|p{0.55cm}p{0.55cm}}
\hline
 & \multicolumn{2}{c|}{AttnGAN} & \multicolumn{4}{c}{DF-GAN} \\ \hline
Settings & \multicolumn{2}{c|}{Task~1: Train} & \multicolumn{2}{c|}{Task~2: Train} & \multicolumn{2}{c}{Task~3: Train}  \\
\hline
\textbf{CUB} & CS$\uparrow$ & FID$\downarrow$ & CS$\uparrow$ & \multicolumn{1}{c|}{FID$\downarrow$} & CS$\uparrow$ & FID$\downarrow$ \\ \hline
Paper & 68.00$^*$ & 23.98$^*$ & - & \multicolumn{1}{c|}{14.81$^*$} & - & -   \\
RM & 68.00 & 23.98 & - & \multicolumn{1}{c|}{14.81} & 58.10$^*$ & 12.10$^*$ \\ \hline
{+Mixup} & 65.82 & 41.47 & 57.29  & \multicolumn{1}{c|}{28.73} & 57.36 & 25.77  \\
 {+DiffAug} & {66.94} & \textbf{22.53} & {58.22}& \multicolumn{1}{c|}{{17.27}} &  {{58.05}} &  {12.35}  \\
{+RandMask} & 67.80 & \textbf{15.59} & 57.96$^*$  & \multicolumn{1}{c|}{15.42} & 58.07 & 15.17 \\
{+Add Noise} & 67.79 & \textbf{17.29}  & 57.46 & \multicolumn{1}{c|}{48.23} & 57.58  & 42.07   \\
\hline
{+$ITA_T$}  & \textbf{68.53$^\dagger$} & \textbf{14.14} & \textbf{58.09} & \multicolumn{1}{c|}{\textbf{14.03}} & \textbf{58.80$^\dagger$} & \textbf{12.17} \\
{+$ITA_T$+$L_{db}$} & \textbf{68.10} & \textbf{14.55} & \textbf{58.07} & \multicolumn{1}{c|}{\underline{\textbf{11.74}}} & \underline{\textbf{58.67}} & \textbf{11.58}  \\
{+$ITA_C$} & \underline{\textbf{68.42}} & \textbf{13.68} & \textbf{58.25} & \multicolumn{1}{c|}{\textbf{12.70}} & \textbf{58.23} & \textbf{11.81}  \\
{+$ITA_C$+$L_{db}$} & \textbf{68.18} & \underline{\textbf{13.74}} & \textbf{\textbf{58.30$^\dagger$}} & \multicolumn{1}{c|}{\textbf{12.93}} & \textbf{58.23} & \underline{\textbf{10.77}}   \\
{+$ITA_C$+$L_r$} & \textbf{68.20} & \textbf{\textbf{13.17$^\dagger$}} & \underline{\textbf{58.27}} & \multicolumn{1}{c|}{\textbf{11.70$^\dagger$}} & \textbf{58.24} & \textbf{10.45$^\dagger$}\\ 
\hline
Settings & \multicolumn{2}{c|}{Task~4: Tune} & \multicolumn{2}{c|}{Task~5: Tune} & \multicolumn{2}{c}{Task~6: Tune} 
\\ \hline
\textbf{COCO} &CS$\uparrow$ & FID$\downarrow$ & CS$\uparrow$ & \multicolumn{1}{c|}{FID$\downarrow$} & CS$\uparrow$ & FID$\downarrow$ \\ \hline
Paper & 50.48 & 35.49 & - & \multicolumn{1}{c|}{19.23} & - & -  \\
RM & 50.48 & 35.49 & 50.94 & \multicolumn{1}{c|}{15.41} & 50.94 & 15.41  \\
+ Tuned & 62.59$^*$ & 29.60$^*$ & 50.63$^*$ & \multicolumn{1}{c|}{15.67$^*$} & 50.71$^*$ & 15.22$^*$  \\  \hline
{+Mixup} & 62.30 & 33.41 & 50.38 & \multicolumn{1}{c|}{23.80} & \textbf{50.83} & 22.86\\
{+DiffAug} & {\textbf{65.44}} & {33.86} & {49.45} & \multicolumn{1}{c|}{{21.31} } & {\textbf{50.94}}  & {18.97}    \\
{+RandMask} & \textbf{63.76} & \underline{\textbf{23.82}} & 50.54 & \multicolumn{1}{c|}{15.74} & 50.64 & 15.33  \\
{+Add Noise} & \textbf{64.77$^\dagger$} & 35.47 & \textbf{50.94$^\dagger$} & \multicolumn{1}{c|}{34.90} & \textbf{50.80} & 33.84  \\ 
\hline
{+$ITA_T$+$L_{db}$} & \textbf{63.31} & \textbf{26.65} & {50.60} & \multicolumn{1}{c|}{\textbf{15.05}} & \textbf{50.77} & \textbf{13.67}\\
{+$ITA_C$+$L_{db}$} & {\textbf{63.97}} & {\textbf{25.82}} & \underline{\textbf{50.92}} & \multicolumn{1}{c|}{\underline{\textbf{14.71}}} & \underline{\textbf{50.98}} & \underline{\textbf{13.28}} \\
{+$ITA_C$+$L_r$} & \underline{\textbf{64.59}} & \textbf{22.70$^\dagger$} & {\textbf{50.81}} & \multicolumn{1}{c|}{\textbf{13.71$^\dagger$}} & \textbf{51.02$^\dagger$} & \textbf{12.49$^\dagger$} \\
\hline
\end{tabular}
%
\caption{CS$\uparrow$ and FID$\downarrow$ for AttnGAN, and DF-GAN with Mixup, Random Mask, Add Noise, and the proposed SADA components
on CUB and COCO.  *: Baseline results; \textbf{Bold}: Results better than the baseline; $^\dagger$: Best results; \underline{Underlines}: Second best results; `RM': Released Model; `e': epochs.
}
\label{tab:results}
\end{center}
\end{table}

\subsubsection{Qualitative Results}
As depicted in Figure~\ref{fig:ablation} and further examples in Supplementary Materials~\ref{app:More_Results}, we derived several key insights.

{1). Semantic collapse happens in the absence of a sufficient $GisC$:}
As seen in Figure~\ref{fig:ablation}, neither non-augmented nor other augmented methods fail to prevent semantic collapse in different backbones. The application of $GisC$ through SADA serves to alleviate this issue effectively. 
We also notice that semantic collapse is more severe when a complex description is given. Applying SADA alleviates the semantic collapse across all descriptions (More results shown in Section~\ref{app:Results: SADA on Complex Sentences and Simple Sentences}).

{2). $ITA$ preserves textual semantics:}
It shows that generated images of models wo/ $ITA$ on $e'_{s|r}$ still maintain the main semantics of $e_{s|r}$ though they have low quality, indicating the textual semantic preservation of $ITA$. 

\begin{figure}[t]
    \centering
    \includegraphics[width=\linewidth]{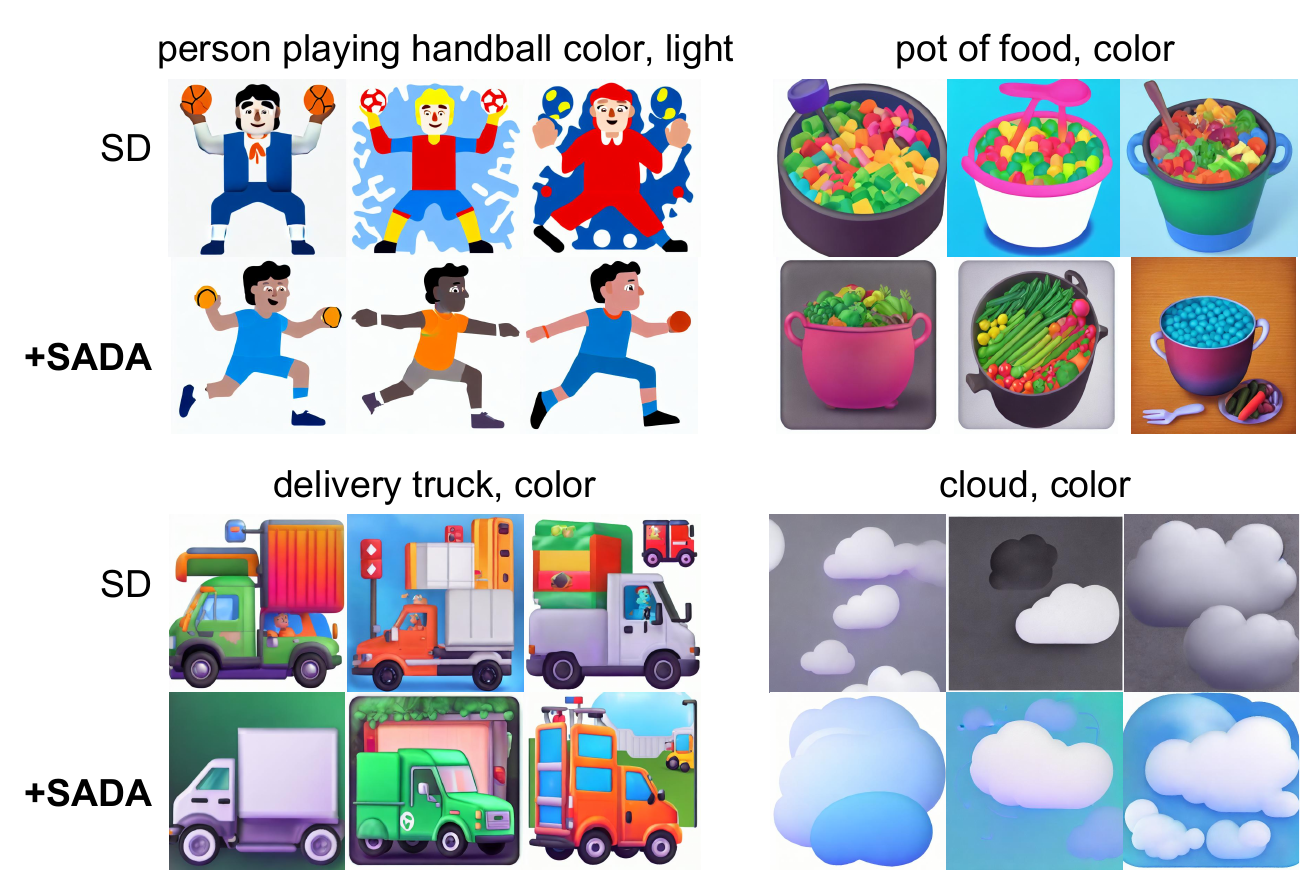}
    \caption{Generated examples of SD tuned on the Emoji dataset wo/ and w/ SADA. A significant improvement in diversity with $+ITA_C +L_r$ can be observed, especially in terms of skin color and view perspective. }
    \label{fig:sd_emoji}
\end{figure}

{3). SADA enhances generated image diversity:}
SADA appears to improve image diversity when input noise is not fixed significantly and $e_{s|r}$ of testing text is used. The greatest improvement in image diversity was achieved by $ITA_C\! +\! L_r$, as the detailed semantics of birds, are more varied than the other semantics. Textual unmentioned details such as skin colors as shown in Figure~\ref{fig:sd_emoji} is more various when using SADA.
Analysis of textual unmentioned details can be observed in Supplementary Materials Figure~\ref{fig:AttnCUB} (highlighting wing bars, color, background).


{4). 
$ITA$ with $GisC$ improves the model generalization by preventing semantic collapse:
Using $ITA_T + L_{db}$ and $ITA_C \!+ \! L_{db}/L_{r}$ lead to obvious image quality improvement when more $Noise$ is given, corresponding to our Proposition~\ref{prop:enhance_image} and Proposition~\ref{prop:lip}. 
However, with $ITA_C + L_{db}$, though the model can produce high-quality images, generated images on $e_{s|r}$ and $e'_{s|r}$ are quite similar while $ITA_C + L_r$ varies a lot, especially in the background, implying a not guaranteed semantic preservation of $L_{db}$ and a tighter constraint of $L_r$ as proved in Proposition~\ref{prop:r_better_db}. Furthermore, $ITA_C + L_r$ provides the best image quality across all experiments. 

\subsubsection{SADA on Complex Sentences and Simple Sentences}
\label{app:Results: SADA on Complex Sentences and Simple Sentences}
We explore the effect of SADA on complex sentences and simple sentences. We use textual embeddings of sentences in Table~\ref{tab:sent} and illustrate interpolation examples at the inference stage between $e_{s|r}$ and $e'_{s|r}$ as shown in Figure~\ref{fig:interpo} and Figure~\ref{fig:vs} right side, where $Noise \sim \mathcal{N}(0, \beta \cdot \mathbb{C}_{ss|r} \mathbb{I})$. 
It can be observed that models trained with SADA can alleviate the semantic collapse that occurs in models without SADA, and its semantics can resist even larger $Noise$ given. 
Using $e'_{s|r}$ at the inference stage can cause image quality degradation, which reveals the robustness of the models.


As shown in Figure~\ref{fig:vs}, on the left side, DF-GAN with SADA generates more text-consistent images with better quality from rough to precise descriptions compared to other augmentations. The Right side indicates that DF-GAN without augmentations experiences semantic collapse when larger $Noise$ is given. The semantic collapse is more severe when a complex description is given. Applying SADA alleviates the semantic collapse across all descriptions. The model with SADA can generate reasonably good and text-consistent images when the $1.5Noise$ with complex description is given. These visualizations further verified the effectiveness of our proposed SADA.

\begin{table}[t]
    \centering
    \begin{tabular}{c|l}
        \hline
        sent1 & this is a yellow bird with a tail. \\ \hline
        sent2 & \begin{tabular}[l]{@{}l@{}} this is a small yellow bird with a  tail \\ and gray wings with white stripes. \end{tabular} \\ \hline
        sent3 & 
        \begin{tabular}[l]{@{}l@{}} this is a small yellow bird with a \\ a grey long tail and gray wings with white stripes. \end{tabular} 
         \\ \hline
    \end{tabular}
    \caption{Rough, detailed, and in-between description used for generation. }
    \label{tab:sent}
\end{table}
\begin{figure}[t]
    \centering
    \includegraphics[width=\columnwidth]{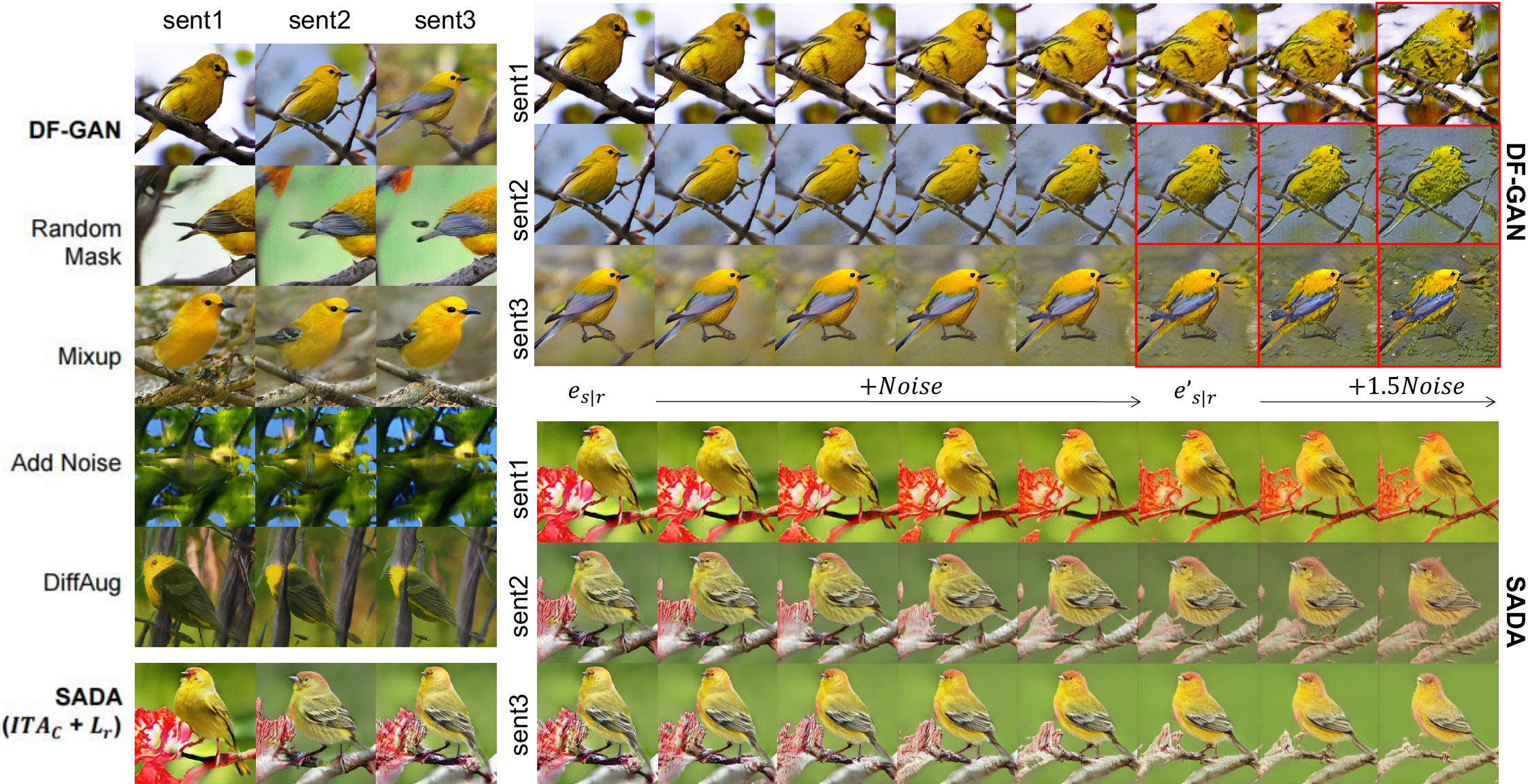}
    \caption{Left: Generated results of DF-GAN with different methods on rough to detailed sentences. Right: 
    Interpolation examples at the inference stage between $e_{r|s}$ and  $e_{r|s}'$ of DF-GAN and it with SADA on rough to detailed sentences. $e_{r|s}'$, input noise for generator $G$, and textual conditions are the same across all rows. Examples with significant collapse are highlighted by red.}
    \label{fig:vs}
\end{figure}





\section{Conclusion}

In this paper, we propose a Semantic-aware Data Augmentation framework (SADA) that consists of $ITA$ (including $ITA_T$ and $ITA_C$) and $L_r$.  We theoretically prove that using $ITA$ with T2Isyn models leads to text-image consistency improvement. We also show that using $GisC$ can improve generated image quality, and our proposed $ ITA_C + L_r$ promotes image quality the most.
ITA relies on estimating the covariance of semantic embeddings, which may, however, be unreliable in the case of unbalanced datasets. We will explore this topic in the future. 


\section*{Acknowledgments}
The work was partially supported by the following: National Natural Science Foundation of China under No. 92370119, No. 62376113, and No. 62206225; Jiangsu Science and Technology Program (Natural Science Foundation of Jiangsu Province) under No.  BE2020006-4;
Natural Science Foundation of the Jiangsu Higher Education Institutions of China under No. 22KJB520039.

\bibliography{ITA}

\newpage
~~~~~~~~~~~~~~~~
\newpage

\appendix

\section{More Mathematical Details}

Here, we provide more details for our derivations and proofs.

\subsection{Derivation Details of Training Objectives for $G$ with $ITA$ }
\label{app:ita_obj}
Based on empirical risk minimization (ERM), the empirical risk for generator $G$ is defined as: 
\begin{align}
    R_k(\theta):= \frac{1}{k} \sum_{i=1}^k L(\theta, X).
\end{align}
Its standard augmented version and corresponding augmented loss are defined as:
\begin{align}
    \hat{R}_k(\theta):= \frac{1}{k} \sum_{i=1}^k \int_\mathcal{A} L(\theta, f(X))d{Q}_{ITA}(f), 
\end{align}
where $\mathbf{Q}_{ITA}$ is a probability distribution on a group $\mathcal{A}$ of $ITA$ transforms from which $f$ is sampled. Since only one $ITA$ will be used, the general sample objective with $ITA$ is defined as:
\begin{align}
\label{eq:itatsolution}
    \min_\theta \Hat{R}_k(\theta):= \frac{1}{k} \sum_{i=1}^k  {L}(\theta, ITA(X_i)).
\end{align}
We then define the solution of Eq.~(\ref{eq:itatsolution}) as:
\begin{align}
    \theta_{ITA}^* \in arg \min_{\theta \in \Theta }\frac{1}{k} \sum_{i=1}^k  {L}(\theta, ITA(X_i)) ,
\end{align}
where $\Theta$ is defined as some parameter space.

\subsection{Proof Details}
\label{app:proof}
\begin{proposition}[$ITA$ increases T2Isyn semantic consistency] 
\label{appprop:semantic_improvement}
    Assume exact invariance holds.
    Consider an unaugmented text-image generator  $\Hat{\theta}(X)$ of $G$ and its augmented version $\Hat{\theta}_{ITA}$. For any real-valued convex loss $S(\theta, \cdot)$ that measures the semantic consistency, we have:
    \begin{align}
    \label{eq:consisteny}
        \mathbb{E} [S(\theta, \Hat{\theta}(X))] \ge \mathbb{E} [S(\theta, \Hat{\theta}_{ITA} (X))] ,
    \end{align}
    which means with $ITA$, a model can have lower $\mathbb{E} [S(\theta, \Hat{\theta}_{ITA} (X)]$ then a better text-image consistency.
\end{proposition}

\begin{proof} 
   From Group-Theoretic Framework for Data Augmentation~\cite{chen2020group}, we obtain a direct consequence that:
   \begin{align}
     Cov [ \Hat{\theta}_{ITA}(X)] \preceq Cov [\Hat{\theta}(X)]\;,
   \end{align}
   where $Cov[\cdot]$ means the covariance matrix decreases in the Loewner order. 
   Therefore for any real-valued convex loss function $S(\theta, \cdot)
   $, we have
   Proposition~\ref{appprop:semantic_improvement} is proofed.

    Empirically, $S(\theta, \cdot)$ can be a real-valued convex loss produced by discriminators, perceptual semantic loss produced by pre-trained models, and others.   
    It also suggests that $ITA$ can be considered as an algorithmic regularization like other data augmentations, and augmented $G$ can obtain better text-image consistency. 
\end{proof}


\begin{proposition}
\label{appprop:close_form}
    $ITA_C$ can be considered a closed-form solution for general textual semantic preserving augmentation methods of T2Isyn. 
\end{proposition}
\begin{proof} 
   Assume exact invariance holds. Captions offered in the dataset are based on real images, thus: $ e_{s|r} \triangleq e_s$. Assume all the models that are mentioned in the following are well-trained.  
   We consider two situations:
   \begin{enumerate}
       \item For methods that use extra models to generate more textual data based on real images $r$ (such as I2T2I~\cite{dong2017i2t2i}, which uses a pre-trained captioning model), we have:
            \begin{align}
            &e_{s|r} \sim \mathcal{N}(m_r, \mathbb{C}_{ss|r}\mathbb{I}) = {Q}_{s|r}, \\
            &e_{s|r}' \sim \mathcal{N}(m_r, \mathbb{C'}_{ss|r}\mathbb{I})  = {Q'}_{s|r}.
            \end{align}
        When the extra models are trained on the dataset used for T2Isyn, exact invariance holds. We have:
           \begin{align}
               &{Q}_{s|r} =_d {Q'}_{s|r}, \\
               &e_{s|r}' \sim \mathcal{N}(m_r, \mathbb{C}_{ss|r}\mathbb{I}).
           \end{align}
        \item Consider methods that use extra models that generate synonymous texts based on real texts (such as retrieving texts from the dataset and refining the conflicts like RiFeGan~\cite{cheng2020rifegan}, using extra pre-trained synonymous text generating model, and our proposed $ITA_T$). Assume exact invariance holds. Captions offered in the dataset are based on real images, thus: $ e_{s|r} \triangleq e_s, e_{s|r} \sim Q_{ss|r} $. 
        Augmented texts $e_s'$ are retrieved from the dataset and refine the semantic conflicts between $e_{s|r}'$ and $e_{s|r}$ based on the main semantics of real images $r$. 
       Therefore:
       \begin{align}
            e_{s|r} \sim N(m_r, \mathbb{C}_{ss|r}\mathbb{I}) = {Q}_{s|r}, \\
           e_{s|r}' \sim N(m_r, \mathbb{C}_{ss|r}\mathbb{I}) ={Q}_{s|r}.
       \end{align}
       \end{enumerate}
              Due to $e_{s|r}$ depending on the semantics of $r$,  $e_{s|r}$  should maintain the main semantics of $r$. Therefore we have:
           \begin{align}
               &m_{s|r} \approx m_{r}, \\
                &e_{s|r}' \sim N(m_{s|r},
                \mathbb{C}_{ss|r}\mathbb{I}),
           \end{align}
           where $ITA_C$ is a closed-form solution.
    Therefore, $ITA_C$ can be considered a closed-form solution for general textual semantic preserving augmentation methods of T2Isyn.
\end{proof}

\begin{proposition}
\label{appprop:enhance_image}
Assume that $E_I$ is linear. Constraining the distribution $Q_\mathcal{E}$ of $e_{f|s}$ can additionally constrain the distribution $\mathcal{F}$ of $f$. 
\end{proposition} 

\begin{proof}
    There are two situations:
    \begin{enumerate}
        \item If $E_I$ is inevitable, Proposition~\ref{appprop:enhance_image} is obvious.
        \item If $E_I$ is not inevitable, 
        constraining $\mathcal{F}$ can affect $\mathcal{E}$ in not Nullspace of $E_I$:$\neg Null(E_I)$:
        \begin{align}
            \mathcal{C}(\mathcal{E}) \propto \mathcal{C}(\neg Null(E_I)(\mathcal{F})),
        \end{align}
        where $\mathcal{C}(\cdot)$ is a certain constraint. 
        For Nullspace, there will be no effect.
        If not all the mass of $\mathcal{F}$ locates in the $Null(E_I)$, Proposition~\ref{appprop:enhance_image} holds. 
        If $\mathcal{F}$ all locates in the $Null(E_I)$ while $E_I$ is well trained, it means $\mathcal{F}$ does not contain any semantics that matches textual semantics, inferring a total collapse of $G$. Since we assume the $G$ can learn the representation, it is impossible that  $\mathcal{F}$ all locates in the $Null(E_I)$.  
    \end{enumerate}
    Therefore, Proposition~\ref{appprop:enhance_image} holds.
\end{proof}


\begin{proposition}
\label{appprop:lip}
$L_r$ leads to $| e_{f|s}'-e_{f|s} |$ is less than or equal to a sequence $\Lambda$ of positive constants, further constrains the semantic manifold of generated embeddings to meet the Lipschitz condition.
\end{proposition}

\begin{proof}
From $L_r = \varphi \cdot || \; ( e_{f|s}'' - e_{f|s}) - {\epsilon^*} \odot \beta \cdot d(\mathbb{C}_{rr|s}) || ^2 \;,$
 we have following constrain for $e_{f|s}'$ and $e_{f|s}$:
\begin{align}
   | e_{f|s}' - e_{f|s} |  \leq | e_{f|s}'' - e_{f|s} |  = |{\epsilon^*}| \odot \beta  \cdot d(\mathbb{C}_{rr|s})\; .   
\end{align}
For each dimension of semantic embeddings, we have:
\begin{align}
\label{eq:lip_2}
   | {e_{f|s}' } ^ d- {e_{f|s}} ^ d |  &=  \beta  \cdot \mathbb{E}[ ( {e_{s|r}'} ^ d - {e_{s|r}} ^ d )^2 ] 
    \;, \\  \notag
     &\leq \beta  \cdot \max [ ({e_{s|r}'} ^ d - {e_{s|r}} ^ d)^2  ] \; \\  \notag
     & =  \beta  \cdot  [ ( {e_{s|r}''} ^ d - {e_{s|r}} ^ d)^2  ] \;, 
     \\  \notag
     & = |{\epsilon^*}^d| \cdot \beta  \cdot d(\mathbb{C}_{rr|s}) ^ d 
     \\  \notag
     & = \beta  \cdot d(\mathbb{C}_{rr|s}) ^ d ,  \\ 
      | {e_{f|s}' } ^ d- {e_{f|s}} ^ d |  & \leq  =  \beta  \cdot d(\mathbb{C}_{rr|s}) ^ d , 
\end{align}
where $d = \{1,...,n\}$ and $n$ is the dimension of the semantic embedding; $d(.)$ represents diagonal part of a matrix; $\beta$ is a positive constant.
Due to the fact of the many-to-many relationship between texts and images,  we have $d(\mathbb{C}_{rr|s}) ^ d > 0$. Assume exact invariance holds,  $|{\epsilon^*}^d| = 1$; $ \beta  \cdot d(\mathbb{C}_{rr|s}) ^ d > 0 $ is a constant.
Thus:
\begin{align}
\label{eq:lip_3}
   | {e_{f|s}' } - {e_{f|s}} |  \leq \Lambda.
\end{align}

If we use $e_{s|r}''$ to generate images, we can alter Eq.~(\ref{eq:lip_3}) to:
\begin{align}
\label{eq:lip_4}
   | {e_{f|s}'' } - {e_{f|s}} |  = \Lambda.
\end{align}

Similar to Eq.~\ref{eq:lip_4}, we can have:
\begin{align}
    | {e_{s|r}'' } - {e_{s|r}} | = \Lambda_s,
\end{align}
where $\Lambda_s$ is also a sequence of positive constants.
Then we have:
\begin{align}
    \frac{| {e_{f|s}'' } - {e_{f|s}} |}{| {e_{s|r}'' } - {e_{s|r}} |}
    = \frac{\Lambda}{\Lambda_s} = M.
\end{align}
Due to the findings that semantic features in deep feature space are usually linearized~\cite{bengio2013better,upchurch2017deep,wang2021regularizing}, 
we assume semantic features for texts and images are linearized. Following: 
 \begin{align}
    \label{eq:shift}
        ( e_{f|s}' - e_{f|s}  ) \propto ( e_{s|r}' - e_{s|r} ).
\end{align}
we can further have that:
\begin{align}
 |\delta| \frac{| {e_{f|s}' } - {e_{f|s}} |}{| {e_{s|r}' } - {e_{s|r}} |} = 
    \frac{| {e_{f|s}'' } - {e_{f|s}} |}{| {e_{s|r}'' } - {e_{s|r}} |}
    = \frac{M}{|\delta| } \le K, \; \text{s.t.} \; e_{s|r}' \neq e_{s|r} ,
\end{align}
where $\delta$ is a non-zero coefficient. 
Finally, $e_{f|s}' = E_I(G(e_{s|r}')), e_{f|s} = E_I(G(e_{s|r}))$ where $E_I$ is the image encoder, we have:
\begin{align}
\frac{| { E_I(G(e_{s|r}')) } - E_I(G({e_{s|r}})) |}{| {e_{s|r}' } - {e_{s|r}} |} \leq K, \; \text{s.t.} \; e_{s|r}' \neq e_{s|r} ,
\end{align}
where it meets Lipschitz condition. 
\end{proof}

\begin{proposition}
\label{appprop:r_better_db}
 $L_r$ provides tighter image semantic constraints than ${L}_{db}$ \cite{gal2022stylegan} which is defined as:
\begin{align}
  {L}_{db} = 1 - \frac{( e_{s|r}'-e_{s|r} ) \cdot (e_{f|s}'- e_{f|s}) }{||( e_{s|r}'-e_{s|r} )||^2 \cdot ||(e_{f|s}'- e_{f|s}) ||^2} \;, 
\end{align}
\end{proposition}

\begin{proof}
For Eq.~(\ref{eq:Lb}),
assume $ {L}_{db} = 0$ and use $\epsilon^*$, combining with Eq.~(\ref{eq:esaug}):
\begin{align}
    e_{s|r}' = e_{s|r}' \sim \phi = e_{s|r} + z =e_{s|r} + \epsilon \odot \beta \cdot  \mathbb{C}_{ss|r} \mathbb{I},
\end{align}

we have:
\begin{align} 
    \frac{( e_{f|s}''- e_{f|s} ) } {|| e_{f|s}''- e_{f|s} ||^2} & = \frac { || e_{s|r}''- e_{s|r} ||^2} {( e_{s|r}''-e_{s|r} )}  \\
    & =\frac{||\beta  \cdot \epsilon^* \odot d(\mathbb{C}_{ss|r})||^2}{ \beta  \cdot \epsilon^* \odot d(\mathbb{C}_{ss|r})} 
     \;  .
\end{align}
Therefore:
\begin{align} 
       | e_{f|s}''-e_{f|s} | =   
       || e_{f|s}''- e_{f|s}||^2 \cdot  \frac{||\beta  \cdot \epsilon^* \odot d(\mathbb{C}_{ss|r})||^2} {  |\beta  \cdot \epsilon^* \odot d(\mathbb{C}_{ss|r}) | } \ge 0 \;.  
\end{align}
where preservation of semantic collapse is not guaranteed due to the distance between $e_{f|s}'$ and $e_{f|s}$ is not contained. This infers that when two slightly semantic distinct textual embeddings are given, the generated images' semantics can also be the same.

Assume $L_r = 0$, we have:
\begin{align}
       | e_{f|s}'' - e_{f|s} | &= |{\epsilon^*}| \odot \beta  \cdot d(\mathbb{C}_{rr|s})\; \\
       & = \beta  \cdot d(\mathbb{C}_{rr|s}) \\
       & > 0 \;,
\end{align}
where provides tighter constraints than $L_{db}$. 
\end{proof}

\section{Algorithms of Applying SADA}
\label{app:algorithms}
The algorithms of SADA can refer to Algorithm~\ref{alg:itac_alg}~and~\ref{alg:itat_alg}.

\begin{algorithm}[t]
\caption{$ITA_C$ algorithm w/ and wo/ $L_r$ in one epoch. Important differences are highlighted as \blue{blue}. Cal. is short for Calculate.} \label{alg:itac_alg}
\begin{algorithmic}[1]
\Require $G$ with parameter $\theta$ for optimization, paired image text encoders $E_I, E_T$, hyperparameters $\beta, \varphi, lr$. Calculated $\mathbb{C}_{ss|r}, \mathbb{C}_{rr|s}$. \hfill $\triangleright$ See Eq.(9)(13)

\For{$\hat{X} = (r,s) \sim \hat{\mathcal{X}}$}
    \State $e_{r|s} \gets E_{I}(r)$, $e_{s|r} \gets E_{T}(s)$
    \State $f \gets G_{\theta}(e_{s|r})$,  $e_{f|s}\gets E_{I}(f)$
    \If { not use $L_{r}$}
        \State $e_{s|r}' \gets e_{s|r} + \blue{\epsilon} \odot \beta \cdot \mathbb{C}_{ss|r}\mathbb{I}, \;  \blue{\epsilon \sim U(-1,1)}$  \hfill $\triangleright$ See Eq.(10)
        \State $f' \gets G_{\theta}(e'_{s|r})$,  $e'_{f|s} \gets E_{I}(f')$
        \State Cal. $L_{ori}$, $L_S$ by using $s, r, f, f', e_{s|r},e'_{s|r}$
        \hfill $\triangleright$ See Eq.(6)
        \State $\theta \gets \theta - lr \cdot \bigtriangledown [L_{ori}+L_S]$
    \ElsIf{use $L_{r}$}
        \State $e_{s|r}'' \gets e_{s|r} + \blue{\epsilon^*} \odot \beta \cdot \mathbb{C}_{ss|r}\mathbb{I}, \; \blue{\epsilon^* \sim \{-1,1\}}$  \hfill $\triangleright$ See Eq.(10)
        \State $f'' \gets G_{\theta}(e_{s|r}'')$, $e_{f|s}'' \gets E_{I}(f'')$
        \State Cal. $L_{ori}$, $L_S$  
        by using $s, r, f, f'', e_{s|r},e''_{s|r}$
        \hfill $\triangleright$ See Eq.(6)
        \State \blue{Cal. $ L_r \! = \!\varphi \! \cdot \! || \; ( e_{f|s}'' \!- \!e_{f|s}) \! -  \!{\epsilon^*} \odot \beta \cdot d(\mathbb{C}_{rr|s}) || ^2 $ $\triangleright$ See Eq.(18)}
        \State $\theta \gets \theta - lr \cdot \bigtriangledown  [L_{ori}+ L_S + \blue{L_r}]$
    \EndIf 
\EndFor
\end{algorithmic}
\end{algorithm}

\begin{algorithm}[h]
\caption{$ITA_T$ algorithm in one epoch. Cal. is short for Calculate.} \label{alg:itat_alg}
\begin{algorithmic}[1]
\Require $G$, $ITA_T$ with parameter $\theta
$, $\alpha$ for optimization, respectively; paired image text encoders $E_I, E_T$, hyperparameters $r$. Calculated $\mathbb{C}_{ss|r}$. \hfill $\triangleright$ See Eq.(9)(13)

\For{$\hat{X} = (r,s) \sim \hat{\mathcal{X}}$}
    \State $e_{r|s} \gets E_{I}(r)$, $e_{s|r} \gets E_{T}(s)$
    \State $f \gets G_{\theta}(e_{s|r})$,  $e_{f|s}\gets E_{I}(f)$
    \State $e_{s|r}' = ITA_T(e_{s|r})$
    \State $f' \gets G_{\theta}(e'_{s|r})$,  $e'_{f|s} \gets E_{I}(f')$
    \State Cal. $L_{ori}$, $L_S$ by using $s, r, f, f', e_{s|r},e'_{s|r}$ \hfill $\triangleright$ See Eq.(6)
    \State $\theta \gets \theta - lr \cdot \bigtriangledown [L_{ori}+L_S]$
    \State Cal. $L_{ITA_T}$ \hfill $\triangleright$ See Eq.(11)
    \State $\alpha \gets\alpha -  lr \cdot \bigtriangledown [L_{ITA_T}]$
\EndFor
\end{algorithmic}
\end{algorithm}


\section{More Experimental Details}
\label{app:More_Experimental_Details}
This section includes implementations of $L_S$ and ${L}_{ITA_T}$  with different backbones.  Parameter settings follow the original models (including augmentations they used) of each framework for all experiments unless specified. For training settings, we train the model from scratch and also use their released model for tuning experiments. 
Notice that we do not conduct $ITA_T$ with $L_r$ because $e_{s|r}'- e_{s|r} \le \epsilon \odot \beta \cdot d(\mathbb{C}_{ff|s})$ is required for $L_r$. 
For all frameworks, we use their original losses $L(\theta, X)$ and $L(\theta, X')$ with $GisC$: $L_{db}$ or $L_r$. See specified parameter settings in Table~\ref{tab:parameters}.
We then demonstrate detailed implementations for tested frameworks. Note that since $ITA_C$ needs no more training, thus model with $ITA_C$ requires no more implementation of  $L_S$   and ${L}_{ITA_T}$.



\subsection{Obtaining $ITA_C$ and $L_r$}

$ITA_C$ and $L_r$ are based on $\mathbb{C}_{ss|r}$ and $\mathbb{C}_{rr|s}$ defined as:
\begin{align}
         \mathbb{C}_{ss|r} = \mathbb{C}_{ss} - \mathbb{C}_{sr}\mathbb{C}_{rr}^{-1}\mathbb{C}_{rs} , \\
        \mathbb{C}_{rr|s} = \mathbb{C}_{rr} - \mathbb{C}_{rs}\mathbb{C}_{ss}^{-1}\mathbb{C}_{sr} \;.
\end{align}
We only used 30K random samples from CUB and COCO training sets, respectively, to obtain $\mathbb{C}_{ss|r}$ and $\mathbb{C}_{rr|s}$ for our experiments. The number of samples follows it of calculating FID~\cite{heusel2017gans}. It is rational to scale the number of samples up according to the size of the dataset. Nevertheless, we do not recommend using the whole training set for the calculation due to its memory consumption. Our calculated $\mathbb{C}_{ss|r}$ and $\mathbb{C}_{rr|s}$ will be released with our code.

\subsection{Applying SADA to GAN-based Methods}

Note that the discriminators are retrained during the tuning process for AttnGAN and DF-GAN since no released checkpoints are available, and we only tune the transformer part of experimental settings, which are specified in paper Table~1. 

\subsubsection{Applying $ITA_T$ to DF-GAN} 
\label{Supplementary:implementation_DF}
DF-GAN~\cite{tao2022df} is a currently proposed one-way output T2Isyn backbone. 
For ${L}_{DF_D} = (\theta_{DF}, \cdot)$ of DF-GAN's Discriminator $ D_{DF}$, we use it as $L_S$ for DF-GAN: 
\begin{align}
    L_{S-DF_D} =  {L}_{DF_D}(\theta_{DF_D},   (e_{s|r},    \mathcal{G}_{DF}(e_{s|r})) + \notag \\  {L}_{DF_D}(\theta_{DF_D},  (e_{s|r}',    \mathcal{G}_{DF}(e_{s|r}'))  + \notag \\ {L}_{DF_D}(\theta_{DF_D},  (e_{s|r},    \mathcal{G}_{DF}(e_{s|r}') 
    )\;,
\end{align}
where $L_{DF}$ is the simplified representation for DF-GAN's original Discriminator losses; $\mathcal{G} = h_{DF}{(G_{DF}(.))}$ where $(.)$ takes a textual embedding, $h_{DF}$ maps generated images of $(G_{DF}$ on the textual embedding. Notations in the following frameworks are similar. 
All embeddings used in DF-GAN are gained from DAMSM images and text encoders. 

Then for generator $G$ loss $L_{DF_G}(\theta_G, \cdot)$, we have loss:
\begin{align}
     L_{S-DF_G} = {L}_{DF_G}(\theta_{DF_G}, (e_{s|r}, \mathcal{G}_{DF}(e_{s|r}))) + \notag \\
      {L}_{DF_G}(\theta_{DF_G}, (e_{s|r}', \mathcal{G}_{DF}(e_{s|r}'))) + \notag \\
      {L}_{DF_G}(\theta_{DF_G}, (e_{s|r}, \mathcal{G}_{DF}(e_{s|r}')))
      . 
\end{align}

Since DF-GAN only uses one discriminator $D_{DF}$ for both semantic matching and image quality supervision. Therefore, we can use $ {L}_{DF_D}(\theta_{DF}, \cdot )$ to force $\mathcal{G}_{DF}(t_s')$ be consistent with $t_s$ by optimizing parameters $r$ of $ITA_T$ : 
\begin{align}
    {L}_{ITA_T - DF} = &  r \cdot {L}_{iemse} (e_{s|r}, e_{s|r}') +  \notag \\ 
    &  (1-r) \cdot {L}_{DF_G}(\alpha,  (e_{s|r},  \mathcal{G}_{DF}(e_{s|r}'))   \;.
\end{align}


\subsubsection{Applying $ITA_T$ to AttnGAN} 
AttnGAN~\cite{xu2018attngan} is a widely used backbone for GAN-based text-to-image generation baseline. 
Since AttnGAN uses both sentence-level $e_{s|r}$ and word-level semantics $e_w$ embeddings, we implement augmented sentence and words $e_{s|r}', e_w'$ as  $e_{s|r}' = ITA_T(e_{s|r}) , e_w' = e_w + (e_{s|r}' - e_{s|r}) $. Other implementations refer to Section~\ref{Supplementary:implementation_DF}.


All embeddings used in AttnGAN are gained from DAMSM images and text encoders.

\subsubsection{Applying $ITA_C$ and $L_r$ to AttnGAN and DF-GAN}
It is easy to apply  $ITA_C$ and $L_r$ to AttnGAN and DF-GAN, by just using augmented textual embeddings for training and using $L_r$ as additional constraining.

\subsubsection{Parameter Settings}
\label{Supplementary:para}
We train each backbone from the start on the CUB dataset and tune their released checkpoint on the COCO dataset. Due to no released checkpoints for discriminators of AttnGan and DF-GAN, we retrain discriminators during the tuning phase. If there is no specification, we follow the original experimental settings of each backbone. 
Specified parameters used for producing final results in the paper are shown in Table~\ref{tab:parameters}. Notice that $\beta$ for $ITA_T$ can be set to zero due to the weak supervision of generative adversarial networks. Specifically, we double the learning rate for $ITA_C + L_r$ tests due to their regularity.

\subsection{Applying SADA to VQ-GAN + CLIP}
We use the released checkpoint and code of \cite{wang2022clip} for tuning. Notice the \cite{wang2022clip} is originally trained on the clip embeddings of images; we directly altered it by using textual CLIP embeddings.
We only tune the transformer part for organizing the discrete code,  while the image-generating decoder part is fixed.
Due to the long training time, we only tune the model for 20 epochs with $L_r$ and use its original $L_{vqclip}(\theta, X)$ for our augmented $X'$ as $L_{vqclip}(\theta, X')$. Other settings follow the original settings. 
All embeddings used in VQ-GAN + CLIP are gained from CLIP images and text encoders. We only test $ITA_C + L_r$ with VQ-GAN + CLIP due to its long training time.


\begin{table}[t]
    \centering
    \begin{tabular}{c|cc}
    \hline 
         Dataset: CUB & \multicolumn{2}{c} {+  $ITA_T$} \\ \hline
         Backbone & Warm-up & $r$  \\ \hline
         AttnGAN &  50 & 0 \\
         DF-GAN &  100 & 0.2 \\ 
         \hline \hline
         Dataset: COCO & \multicolumn{2}{c} {+  $ITA_T$} \\ \hline 
         Backbone & Warm-up & $r$ \\\hline
         AttnGAN &  0 & 0 \\
         DF-GAN &  0 & 0.2 \\   
         \hline \hline
         Dataset: CUB & \multicolumn{2}{c} {+  $ITA_C$} \\ \hline 
         Backbone & $\beta$ & $r$ \\\hline
         AttnGAN &  0.05 & 0 \\
         DF-GAN &  0.05 & 0.2 \\ 
         \hline \hline
         Dataset: COCO & \multicolumn{2}{c} {+  $ITA_C$} \\ \hline 
         Backbone & $\beta$ & $r$ \\ \hline
         AttnGAN &  0.01 & 0 \\
         DF-GAN &  0.01 & 0.2 \\
         VQ-GAN + CLIP &  0.05 & 0.2 \\ 
         \hline \hline
         Dataset: CUB & \multicolumn{2}{c} {+  $ITA_C$ + $L_r$} \\ \hline
         Backbone & $\varphi$ \\ \hline
         AttnGAN &  0.01  \\
         DF-GAN &  0.01  \\ 
         \hline \hline
         Dataset: COCO & \multicolumn{2}{c} {+  $ITA_C$ + $L_r$} \\ \hline
         Backbone & $\varphi$ & Learning Rate  \\ \hline
         AttnGAN &  0.001 & As original \\
         DF-GAN &  0.001 & Doubled \\
         VQ-GAN + CLIP &  0.05  & As original \\ 
         \hline
    \end{tabular}
    \caption{Parameters for experiments.}
    \label{tab:parameters}
\end{table}


\subsection{Applying SADA to Conditional DDPM}
{For conditional DDPM~\cite{ho2020denoising}, $ITA$ should be applied to conditional embeddings (including textual conditional embeddings).
The $GisC$ should be applied to features of generated images at each step.  
}


{Specifically, experiments based on the conditional DDPM, specifically utilizes the MNIST dataset~\cite{deng2012mnist}. The methodology applied involved incorporating our proposed $ITA_C$ on condition embeddings, with further integration of $L_r$ on calculated feature shift of generated images from U-Net's bottleneck.
We first train the bare DM and then use its bottleneck's hidden feature as $e_{s|r}$ and the bottleneck's hidden feature of the next step as $e_{f|s}$. Then other details will be the same as aforementioned.
}


Especially, $\mathbb{C}_{ss|r}$ for $ITA_C$ and $\mathbb{C}_{rr|s}$ for $L_r$ are calculated on the training set using the encoders of the framework. 
We use 30K random samples from each dataset in our experiments. 
Limited sampling also leads to the possible implementation of $ITA_C$ and $L_r$ on super-large datasets.


\subsection{Applying SADA to Stable Diffusion}

We apply $ITA_C$ to Stable Diffusion (SD) by adding $\epsilon \odot \beta \cdot \mathbb{C}_{ss|r}\mathbb{I}$ to 
textual embeddings.  $\beta$ for $ITA_C$ is set to $0.001$. All
SADA applications, including applying $ITA_C$ to Stable 
Diffusion can be referred to as Alg.~\ref{alg:itac_alg} and 
Alg.~\ref{alg:itat_alg}, where $L_{ori}$ is the originally used loss of the applied text-to-image generation model.

\textbf{SD tuning experiments settings:} 
For better verification, we chose datasets that have an obvious domain gap or domain shift with the SD's original training set.
We utilized the Pok\'emon BLIP captions~\cite{pinkney2022pokemon}  as our tuning dataset, comprising 833 paired images and texts that were generated by using BLIP~\cite{li2022blip}. LoRA~\cite{hu2021lora} was employed as our tuning tool, and both experiments shared identical parameter settings, including learning rate, training epochs, batch size, and inference steps.   
The goal of these experiments was to enable the tuned SD model to generate diverse Pok\'emon-style drawings. Following the submitted paper, we employed  CLIPScore (CS)and FID as evaluation metrics. It is worth noting that the limited size of the original dataset led to a relatively large FID score. 
{For CS, as the tuning goal is to generate Pok\'emon style drawings, we use the average embedding of the given text and the sentence `this is a Pok\'emon' because most given text only contains attribute descriptions and does not specify  the Pok\'emon. }
Similarly, we use one additional open-source dataset, Emoji~\footnote{Avaliable at \url{https://github.com/microsoft/fluentui-emoji}} dataset that contains $7.56$K samples, to test tuning with SD with SADA. Corresponding results can be seen in Supplementary Materials~\ref{app:More_Results_SD_more}.


\subsection{$ITA_T$ Implementation Suggestions}
Training $ITA_T$ with adversarial models needs concern about how to avoid exploding gradient. Because in the early stage, the discriminators may not provide meaningful semantic bounding on $ITA_T$, causing the augmented $e_{s|r}'$ located too far from $e_{s|r}$ and then a too large loss for generators which cannot be optimized. 
Thus we suggest a warmup phase before training $ITA_T$. For AttnGAN, we set a warmup phase to avoid this kind of crush. Due to DF-GAN using hinge losses, which cannot be larger than one, it can have no warmup phase. Refers to Table~\ref{tab:parameters} for more parameter details.
We also suggest scaling the learning rate up when training with $ITA$ with $L_{db}$ or $L_r$ due to their regularity.

\subsection{Implementing other augmentations with AttnGAN and DF-GAN}
\textbf{Random Mask.} 
We randomly mask $15\%$ tokens and use the original settings of AttnGAN and DF-GAN. 
AttnGAN with Random Mask collapsed multiple times during the training. We use the checkpoints that were saved before the collapse to resume the training.

\textbf{Random Noise.} 
We sample random noise from Gaussian Distribution and add it back to textual embeddings. Note that the noise scale is the same for AttnGAN and DF-GAN due to they use the same textual encoder (DAMSM).

\textbf{Mixup and DiffAug.}
We use the official code of Mixup~\cite{zhang2017mixup} and DiffAug~\cite{zhao2020differentiable} for our experiments. Augmented images and mixed textual embeddings are used for model training. All model settings follow the original settings of AttnGAN and DF-GAN.



\section{More Results and Analysis}
\label{app:More_Results}

\subsection{More Results of Tuning CLIP}
\label{app:More_Results_clip}
Table~\ref{tab:retrieval} shows additional retrieval task results using varying amounts of training data with a consistent testing set, further validating the efficacy of $ITA$. The results highlight $ITA$'s adaptability across various training set scales, especially smaller ones.

\begin{table*}[t]
    \centering
    \begin{tabular}{l|cc|cc|cc|c}
    \hline
        CLIP  & \multicolumn{2}{c|}{1280} & \multicolumn{2}{c|}{64,000}  &\multicolumn{2}{c|}{ All (118,287)} & Used samples
    \\ 
        & wo/ $ITA$ & w/ $ITA$ & wo/ $ITA$ & w/ $ITA$ & wo/ $ITA$ & w/ $ITA$ & I:image;T:text
    \\ 
    \hline 
       30.40 &  36.02 & \textbf{37.28}+1.26 & 40.76 & \textbf{41.08}+0.32  & 44.43 & \textbf{44.88}+0.45 & IR top1   
    \\ 54.73 &  62.54 & \textbf{63.74}-1.20 & 67.74 & \textbf{68.34}+0.60  & 72.38 & \textbf{72.42}+0.04 & IR top5  
    \\ 
    \hline
    49.88 &  50.90 & \textbf{52.74}+1.84 & 57.92 & \textbf{58.58}+0.96  & 61.20 & \textbf{62.76}+1.56 & TR top1  
    \\ 74.96 &  76.22 & \textbf{76.84}+0.62 & 81.68 & \textbf{82.44}+0.76  & 85.16 & \textbf{85.38}+0.22 & TR top5   
    \\ \hline
      & & + 1.23& & + 0.66 & & +0.57 & Avg.   
    \\ \hline
    \end{tabular}%
\caption{ Retrieval (R) tasks use various numbers of training data.}
\label{tab:retrieval}
\end{table*}

\subsection{More Results of Tuning Stable Diffusion}
\label{app:More_Results_SD_more}

We use one additional open-source dataset, the Emoji dataset that contains $7.56$K samples, to test tuning with SD with SADA. Quantitative results can be seen in Table~\ref{tab:sd_res_emoji} and qualitative results can be seen in Figure~\ref{fig:sd_emoji}. Similar to tuning results on the Pok\'emon BLIP captions dataset, tuning SD with SADA brings improvements in CS and FID. 
Specifically, with SADA, the diversity of generated images is improved. As shown in Figure~\ref{fig:sd_emoji} the left top group, SD tuned with SADA generates person with various skin colors. 

\begin{table}[t]
    \centering
    \begin{tabular}{c|cc|c|cc}
        \hline
         & CS & FID &  & CS & FID \\ \hline
        SD Tuned &  63.28  & 71.44 &+SADA & \textbf{63.44} & \textbf{68.33} \\ \hline
    \end{tabular}
    \caption{Tuning results of SD on  Emoji dataset. Results better than the baseline are highlighted as \textbf{bold}.}
    \label{tab:sd_res_emoji}
\end{table}

\subsection{$ITA_T$ with different $r$ }
\label{app:More_Results_ITA_T_r}
As stated, $r$ can control the augmentation strength.
Larger $r$ in $ITA_T$ leads to more intensive augmentation. However, as shown in Table~\ref{tab:alpha_results}, an inappropriate large $r$ can cause model collapse because $S(\alpha, (e_{s|r}, \mathcal{G}(e_{s|r}'))$ will lose its constraint, causing that $e_{s|r}'$ is too different from $e_{s|r}$ (i.e., $e_{s|r}'$ cannot maintain the main semantics of $e_{s|r}$). Collapse examples are shown in Figure~\ref{fig:itat_collapse}. It can be seen that using $\alpha = 0.3$. $ITA_T$ cannot produce a semantic maintaining $e_{s|r}'$ for $G$.
Within the appropriate range, larger $r$ offers better text-image consistency and image quality.

\begin{figure}[t]
    \centering
    \includegraphics[width=\linewidth]{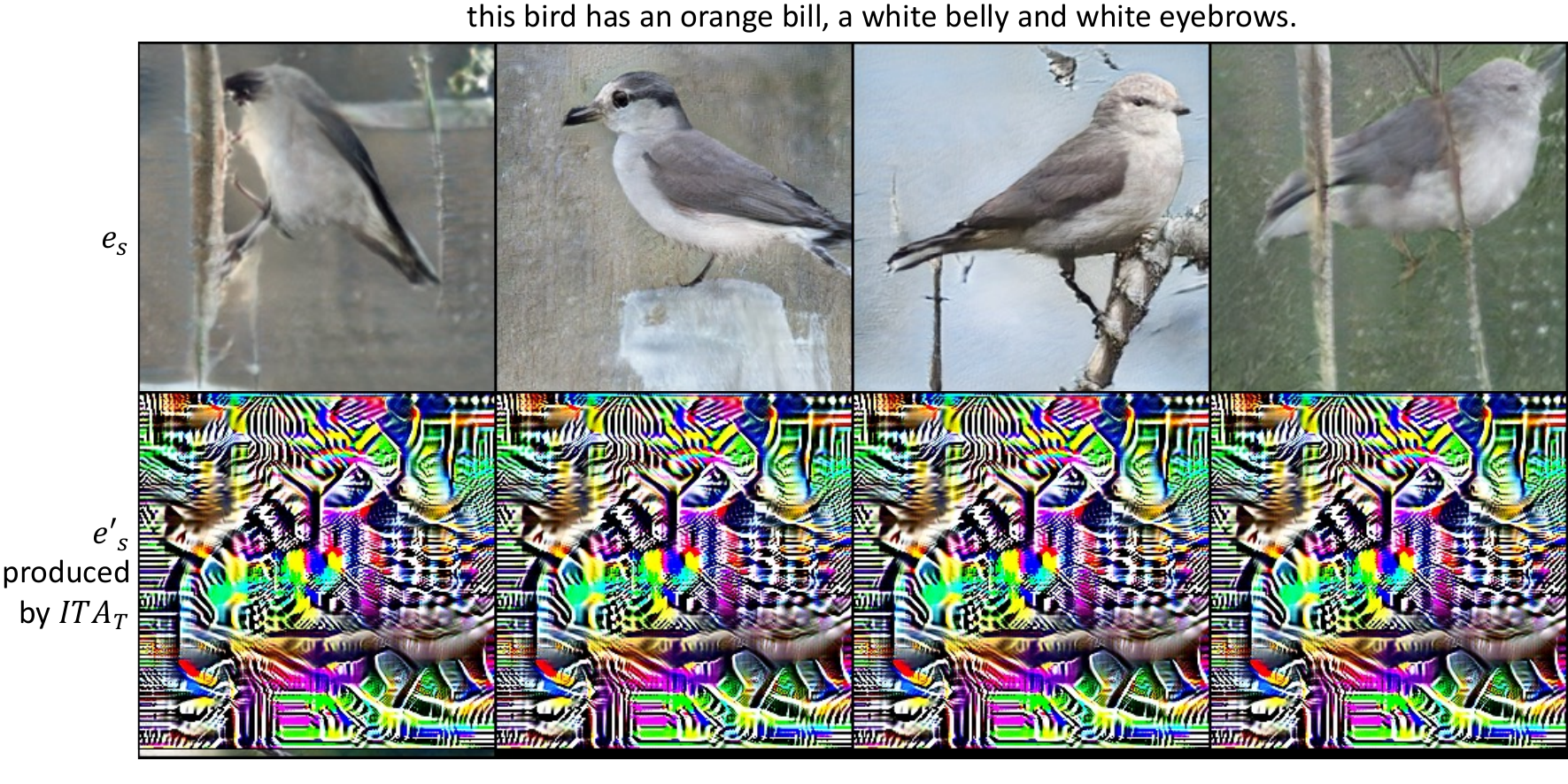}
    \caption{Collapse examples of DF-GAN with $ITA_T + L_{db}$ using $\alpha = 0.3$ generated on $e_{s|r}$ and augmented $e_{s|r}'$.}
    \label{fig:itat_collapse}
\end{figure}





\begin{table}[t]
    \centering 
    \caption{Results of DF-GAN with $ITA_T + L_{db}$ using different $\beta$ values on the CUB dataset, training within 600 epochs.}
    \begin{tabular}{c|cc}
    \hline
        $\beta$ & CS & FID   \\ \hline
         0 & 57.91 & 13.96 \\
         0.1 & 57.93 & 12.7 \\
         0.2 & \textbf{58.07} & \textbf{11.74} \\
         0.3 &  \multicolumn{2}{c}{$ITA_T$ collapses }\\ 
    \hline
    \end{tabular}
    \label{tab:alpha_results}
\end{table}

\subsection{More Results of Different Frameworks with SADA}

We show more generated examples and visualizations of different frameworks.
\begin{itemize}
    \item \textbf{AttnGAN}~\cite{xu2018attngan}: CUB results in Figure~\ref{fig:AttnCUB} and COCO results in Figure~\ref{fig:AttnCOCO_more}. The diversity improvement is more obvious for AttnGAN.
    
    \item \textbf{DF-GAN}~\cite{tao2022df}:  CUB results in Figure~\ref{fig:DFCUB_more} and COCO results in Figure.~\ref{fig:DFCOCO_more}. Semantic collapse can be observed when the model is trained with other augmentations.
    
    \item \textbf{VQ-GAN + CLIP}~\cite{wang2022clip}: COCO results in Figure~\ref{fig:CLIP_more}. Significant semantic consistency can be observed, as missing objects in the generated results of the model that tuned without SADA appear in the model that tuned  with SADA.
    
    \item  \textbf{DDPM}~\cite{ho2020denoising}: Figure~\ref{fig:ddpm_ori}-\ref{fig:ddpm_aug}. Semantic collapse can be observed when the model is trained without SADA.
    
    \item \textbf{SD}~\cite{rombach2021highresolution}:  Pok\'emon-like BLIP tuning examples in Figure~\ref{fig:sd_res} and training loss in Figure~\ref{fig:sd_loss}. The training loss of two experiments can be seen in Figure~\ref{fig:sd_loss}. It can be observed that the coverage state tuning with SADA achieves a lower training loss than without it. We present more qualitative results in Figure~\ref{fig:sd_res}. It can be seen that with SADA, generated images of the tuned model exhibit a more  Pok\'emon-like cartoon style. Emoji tuning results in Figure~\ref{fig:sd_emoji} also reveal the effectiveness of SADA.
\end{itemize}



\subsection{More Results of Other Augmentations}
We show more generated examples and visualizations of different backbones with different augmentations settings.
\begin{itemize}
    \item AttnGAN with different augmentations and SADA: Figure-\ref{fig:AttnCUB} and \ref{fig:AttnCOCO_more}.
    \item DF-GAN with different augmentations and SADA: Figure~\ref{fig:interpo}, Figure~\ref{fig:vs},
    Figure~\ref{fig:DFCUB}, Figure~\ref{fig:DFCOCO}, and 
    Figure~\ref{fig:diffaug}.
\end{itemize}

\subsection{More Results of Ablation Studies of SADA}
The generated examples of the framework applied with different components of SADA can be seen in: 
\begin{itemize}
    \item Ablation Studies on AttnGAN: Figure~\ref{fig:AttnCUB},  Figure~\ref{fig:AttnCOCO_more}.
    \item Ablation Studies on DFGAN: Figure~\ref{fig:DFCUB}, Figure~\ref{fig:DFCUB_more}, Figure~\ref{fig:DFCOCO}, Figure~\ref{fig:DFCOCO_more}.
\end{itemize}

\subsection{Results of Interpolation between $e_{s|r}$ to $e'_{s|r}$}
The interpolation between $e_{s|r}$ to $e'_{s|r}$ can be seen in following figures:
\begin{itemize}
    \item Dense interpolation, across different augmentation methods: Figure~\ref{fig:interpo} and Figure~\ref{fig:vs}.
    
    \item Interpolation, across different augmentation methods: Figure~\ref{fig:DFCUB}, Figure~\ref{fig:DFCOCO},
    and Figure~\ref{fig:diffaug}
\end{itemize}


\begin{figure*}[t]
    \centering
    \includegraphics[width=0.9\linewidth]{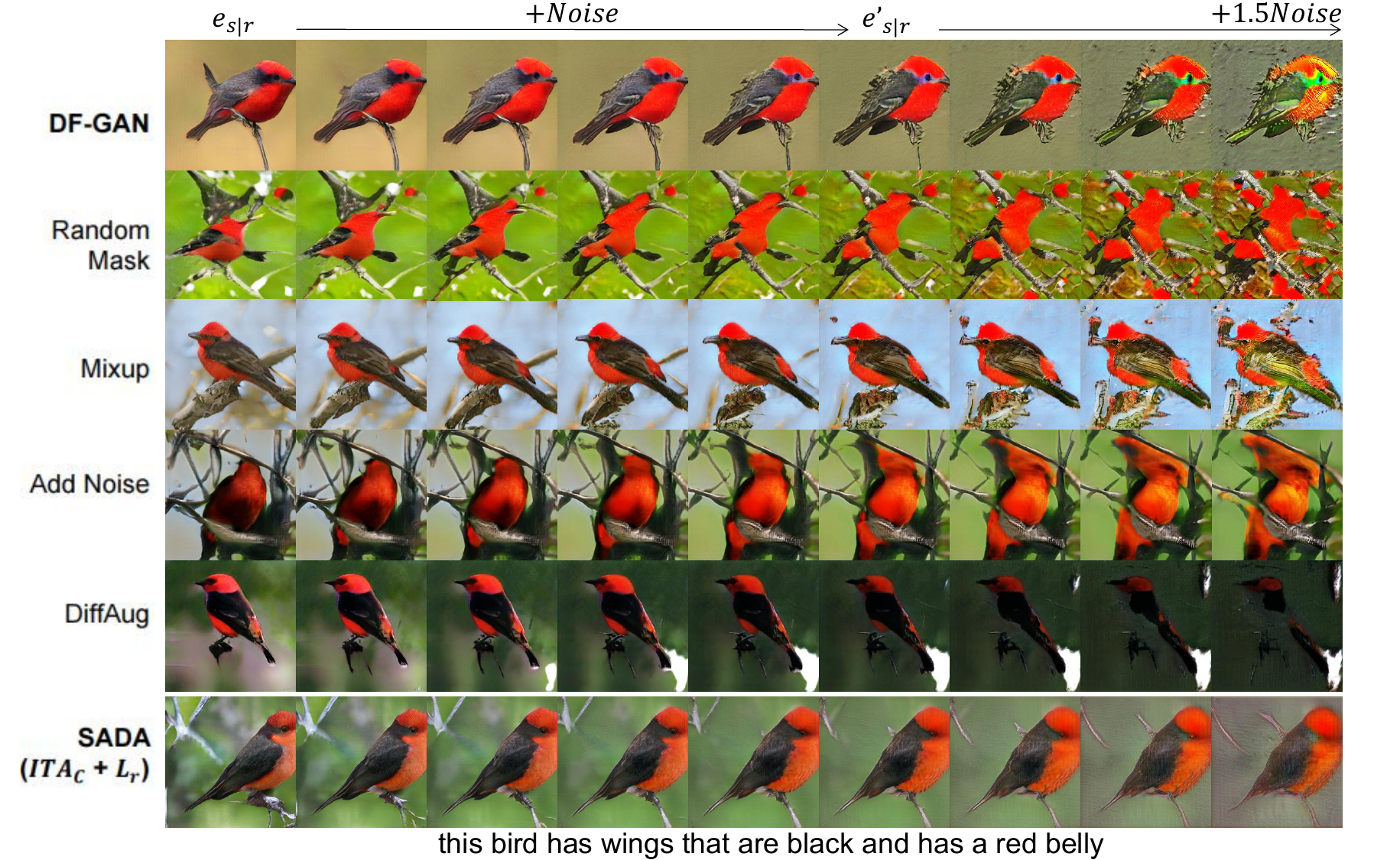}
    \caption{Interpolation examples at the inference stage between $e_{r|s}$ and  $e_{r|s}'$ of DF-GAN with different augmentations methods. $e_{r|s}'$, input noise for generator $G$, and textual conditions are the same across all rows. }
    \label{fig:interpo}
\end{figure*}

\begin{figure*}[t]
\centering
\includegraphics[width=\linewidth]{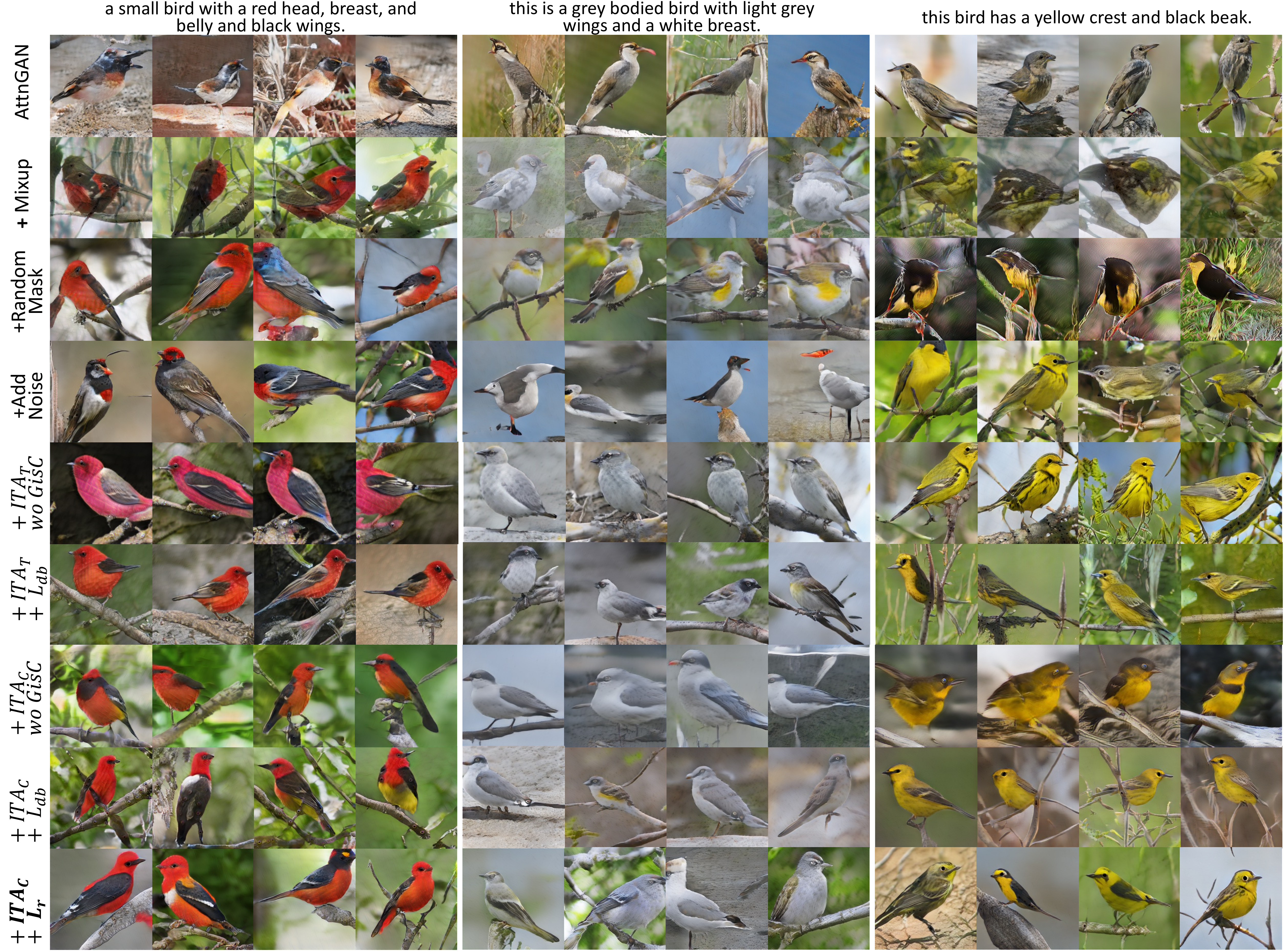}
\caption{Generated results of AttnGAN on CUB. A significant improvement in diversity with $+ITA_C +L_r$ can be observed, especially in terms of color pattern,  backgrounds, and undescribed attributes such as wing bars. }
\label{fig:AttnCUB}
\end{figure*}

\begin{figure*}[t]
    \centering
    \includegraphics[width=0.75\linewidth]{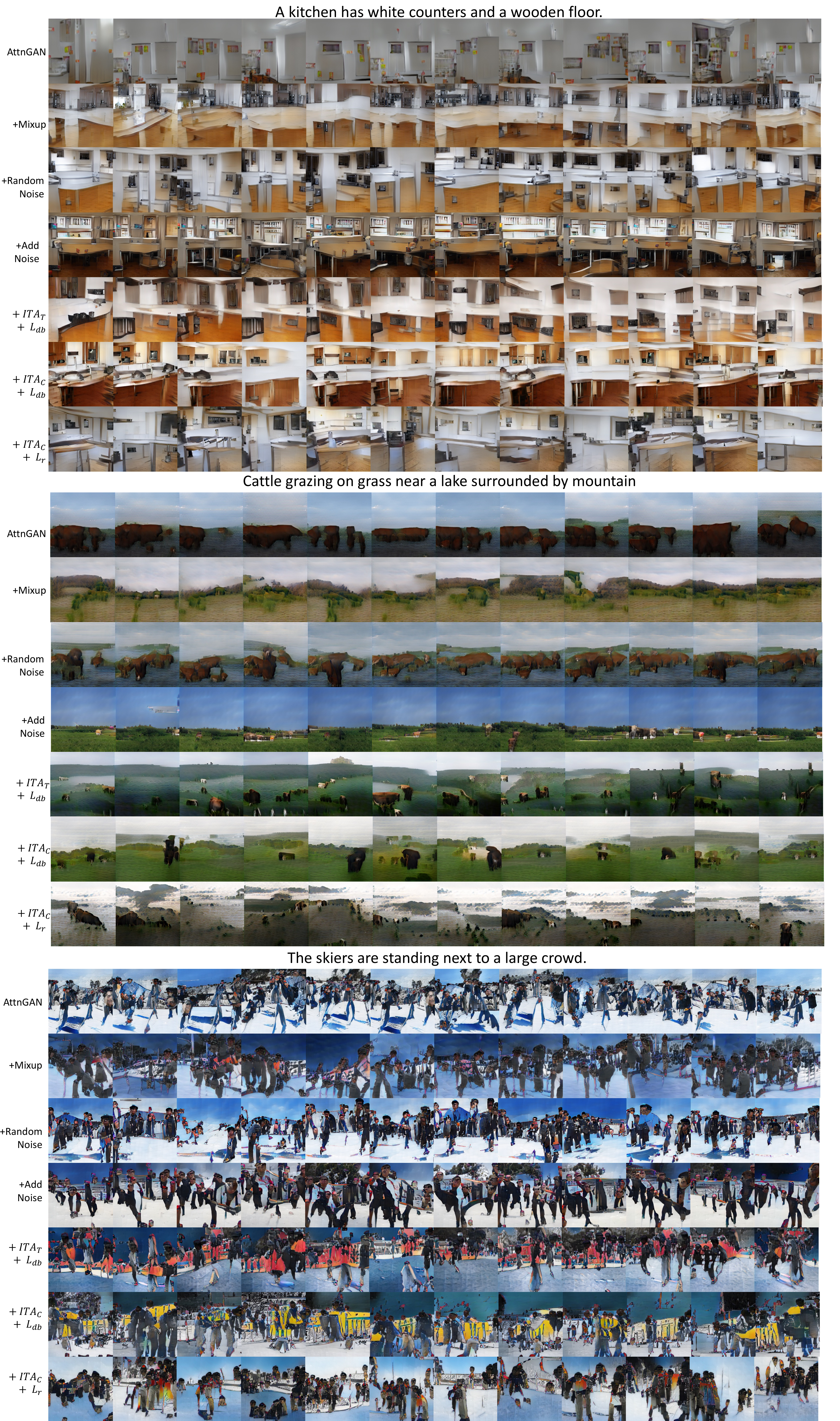}
    \caption{Generated results of AttnGAN on COCO.}
    \label{fig:AttnCOCO_more}
\end{figure*}

\begin{figure*}[t]
    \centering
    \includegraphics[width=0.75\linewidth]{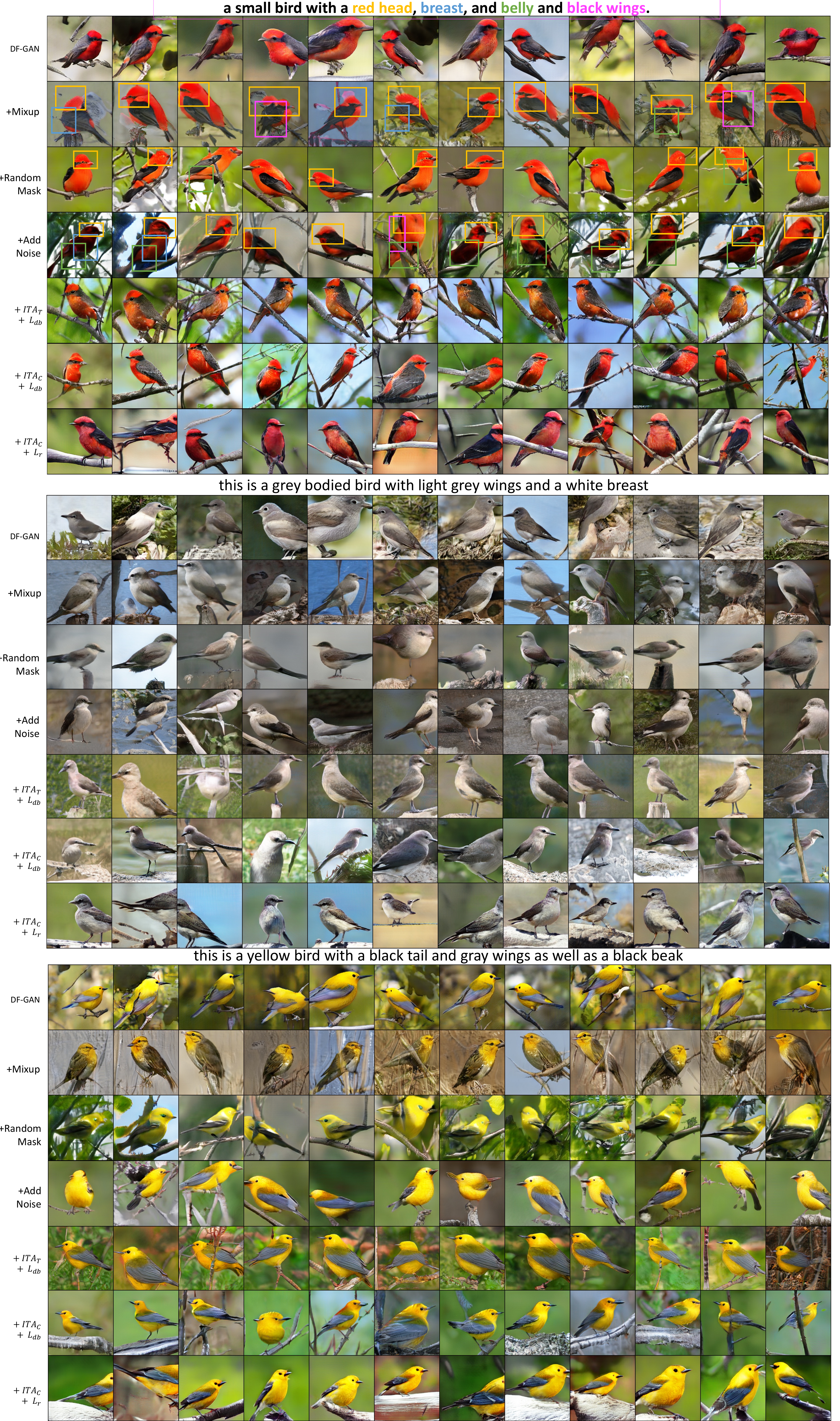}
    \caption{Generated results of DF-GAN on CUB. Semantic mismatches and image quality degradation are highlighted for generated results of DF-GAN w/ Mixup, w/ Random Mask, and w/ Add Noise in the top group. }
    \label{fig:DFCUB_more}
\end{figure*}

\begin{figure*}[t]
    \centering
    \includegraphics[width=0.75\linewidth]{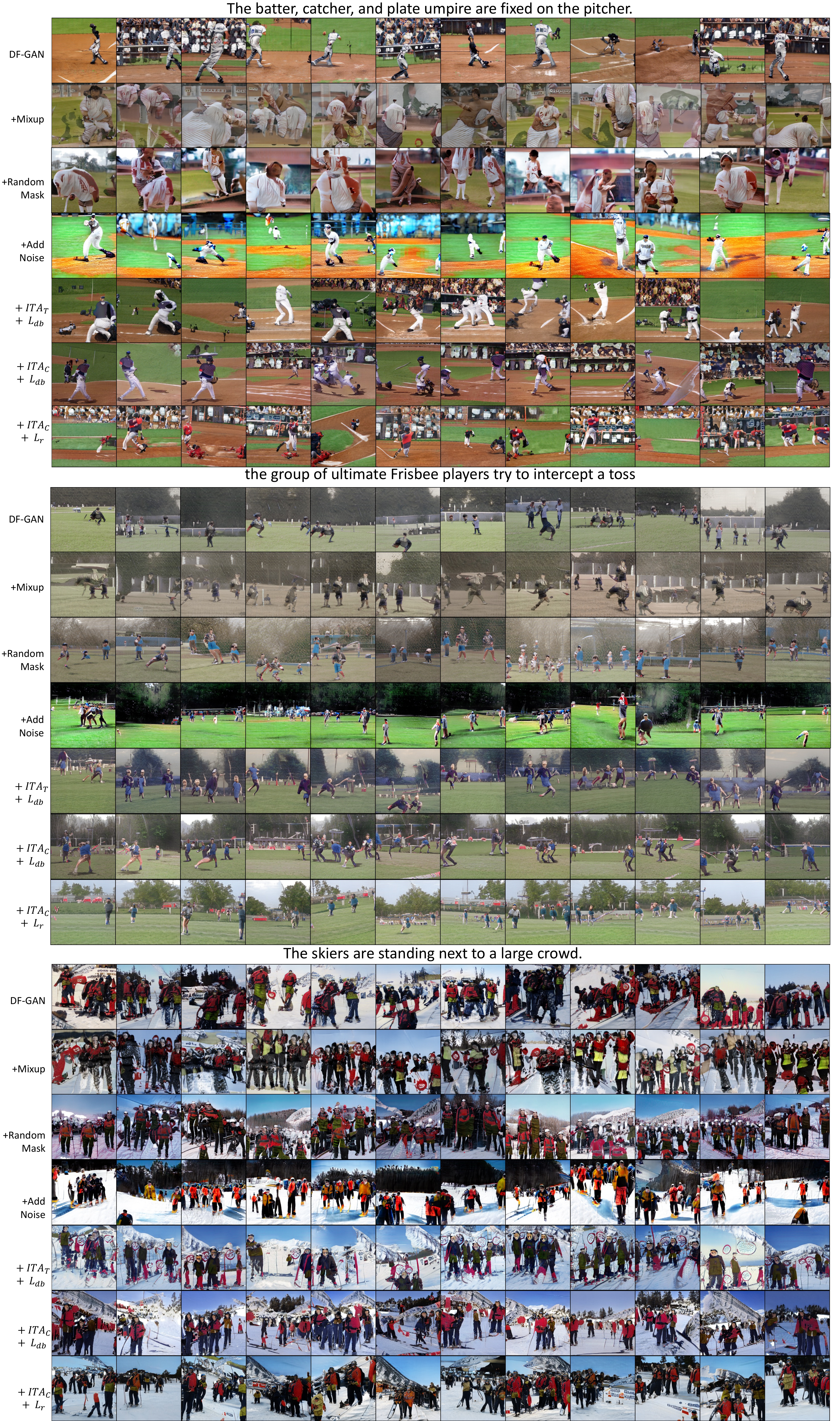}
    \caption{Generated results of DF-GAN on COCO.}
    \label{fig:DFCOCO_more}
\end{figure*}

\begin{figure*}[t]
    \centering
    \includegraphics[width=0.75\linewidth]{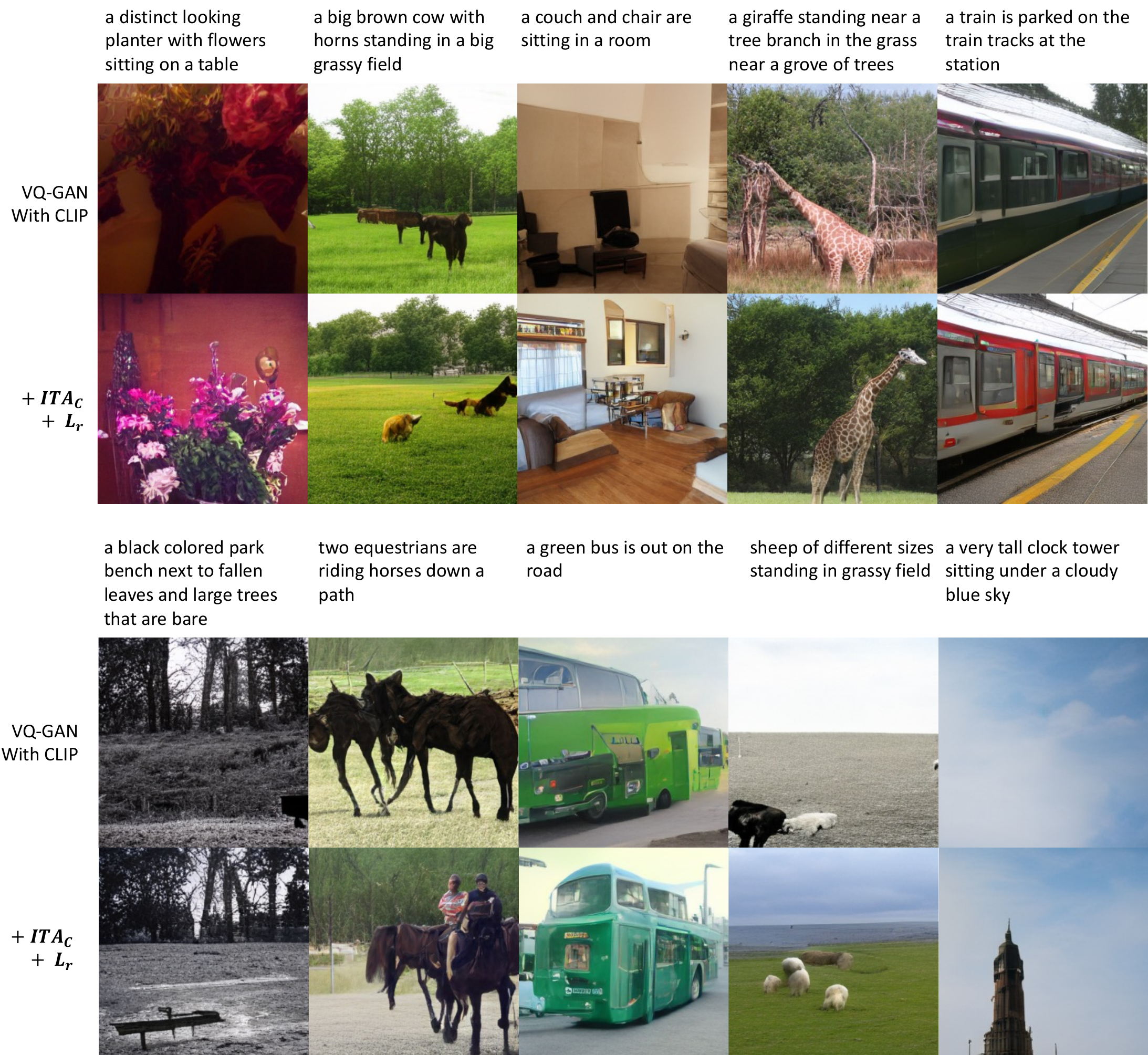}
    \caption{Generated results of VQ-GAN + CLIP on COCO. Significant text-image consistency can be observed.}
    \label{fig:CLIP_more}
\end{figure*}

\begin{figure*}[t]
    \centering
    \includegraphics[width=\linewidth]{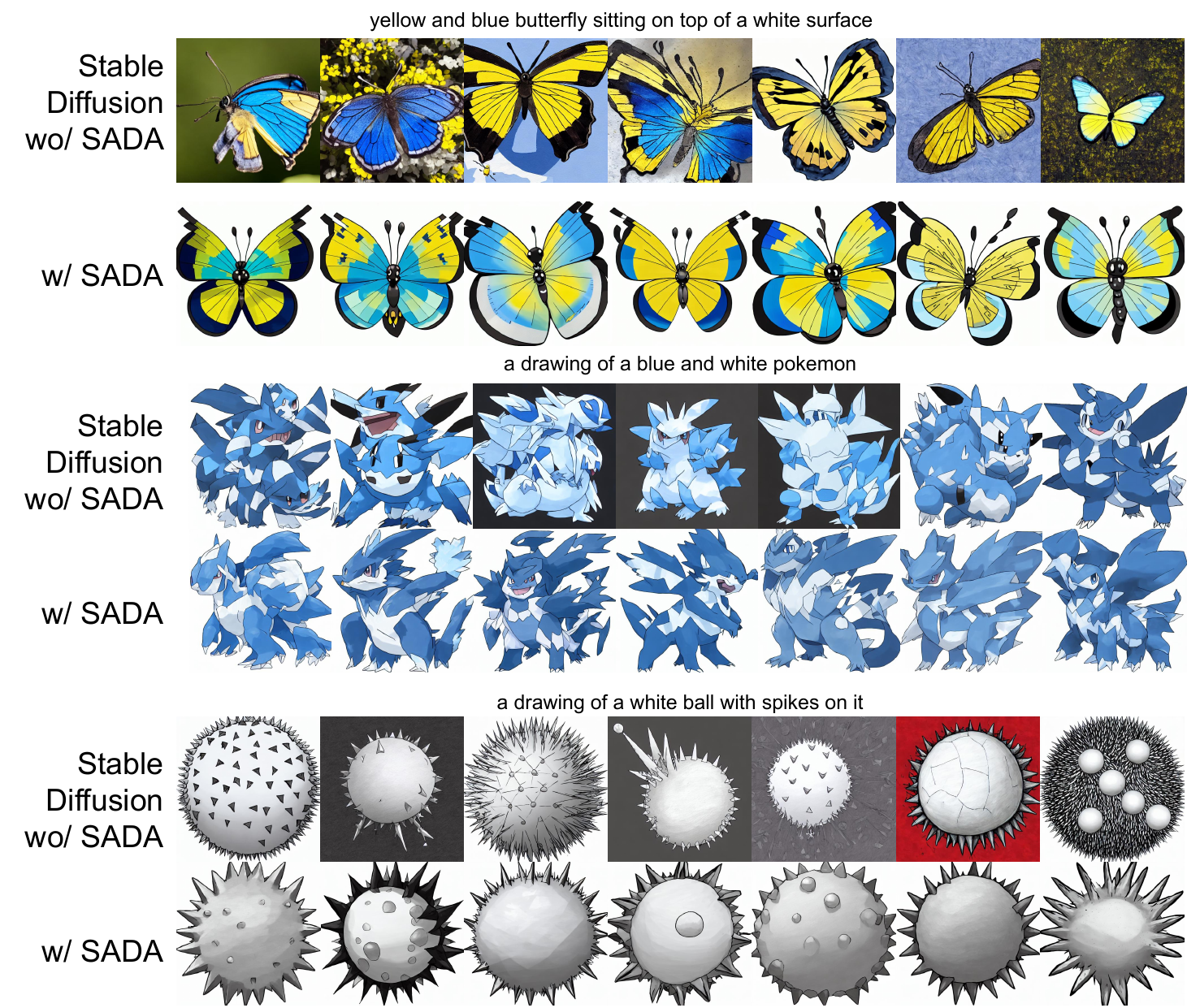}
    \caption{Generated examples of  SD LoRA tuning wo/ and w/ SADA on the Pok\'emon BLIP caption dataset. With SADA, generated images have better image quality and they are more Pok\'emon style-like.}
    \label{fig:sd_res}
\end{figure*}

\begin{figure*}[t]
    \centering
    \includegraphics[width=\linewidth]{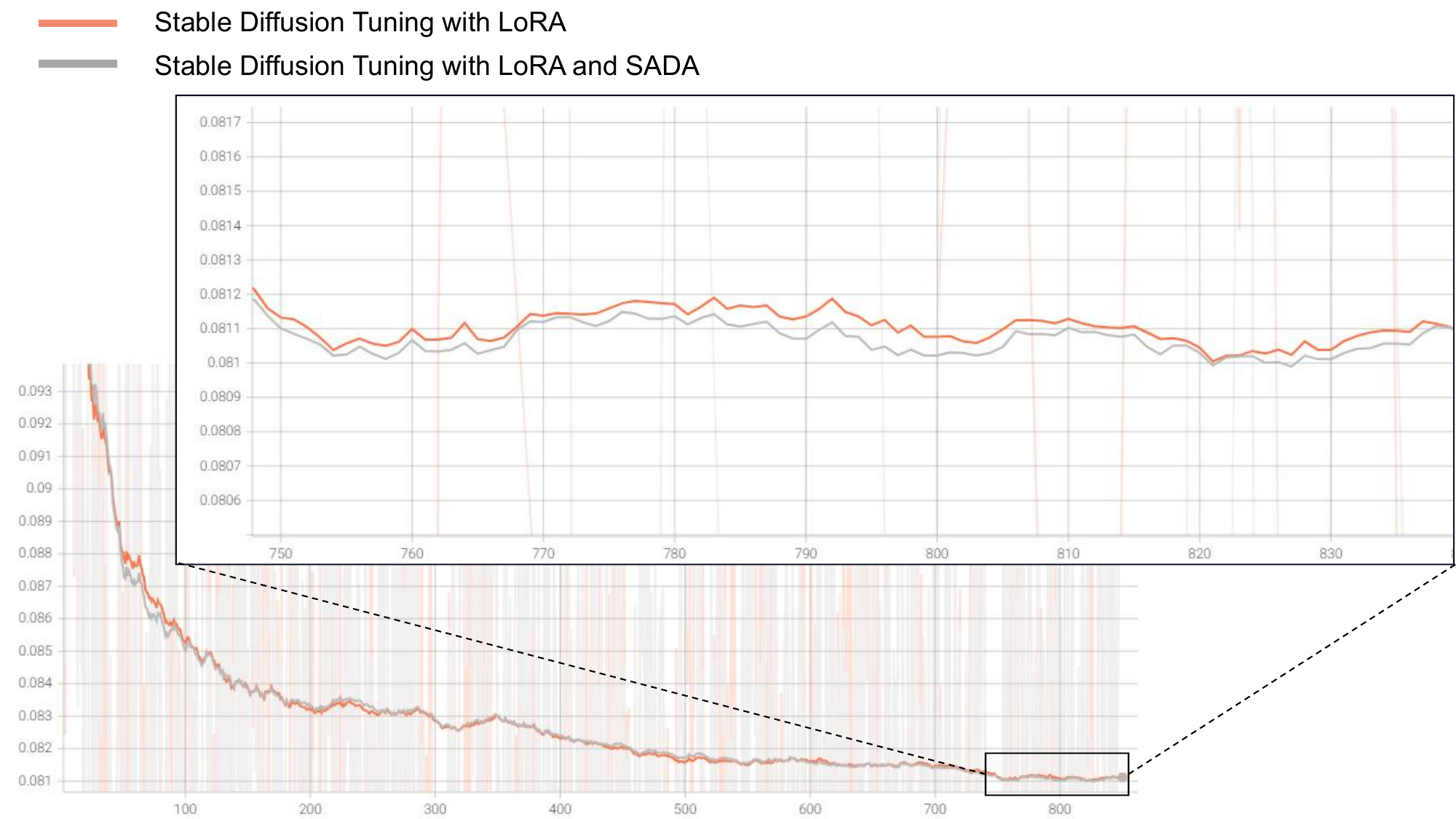}
    \caption{Loss during SD tuning.}
    \label{fig:sd_loss}
\end{figure*}


\begin{figure*}[!t]
\begin{center}
\centerline{
\includegraphics[width=0.75\linewidth]{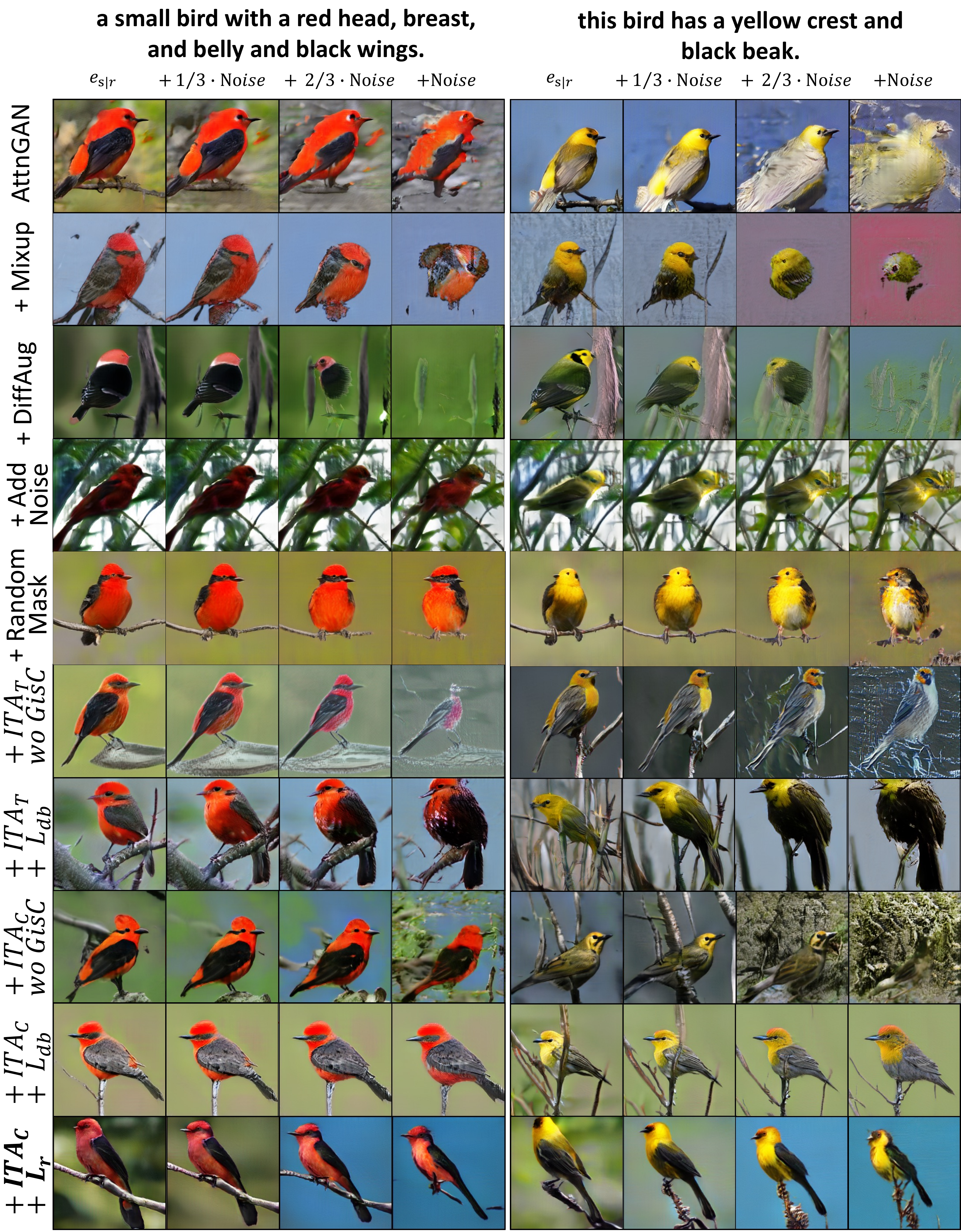}
}
\caption{Generated results of DF-GAN on CUB.}
\label{fig:DFCUB}
\end{center}
\end{figure*}

\begin{figure*}[!t]

\begin{center}
\centerline{
\includegraphics[width=0.75\linewidth]{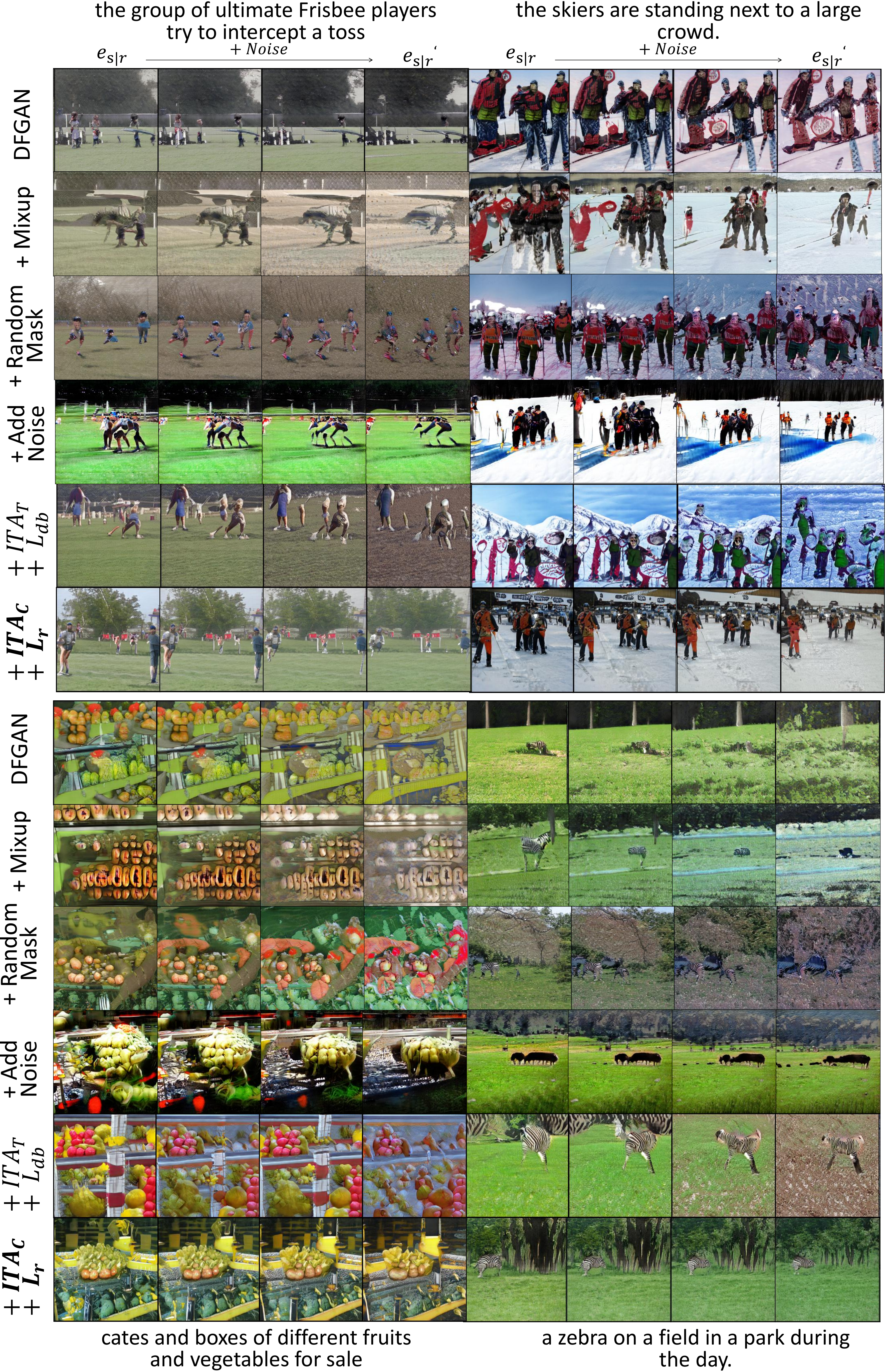}
}
\caption{Generated results of DF-GAN on COCO. Our $ITA_C + L_r$ prevents semantic collapse mostly.}
\label{fig:DFCOCO}
\end{center}
\end{figure*}

\begin{figure*}[!t]
    \centering
    \includegraphics[width=\linewidth]{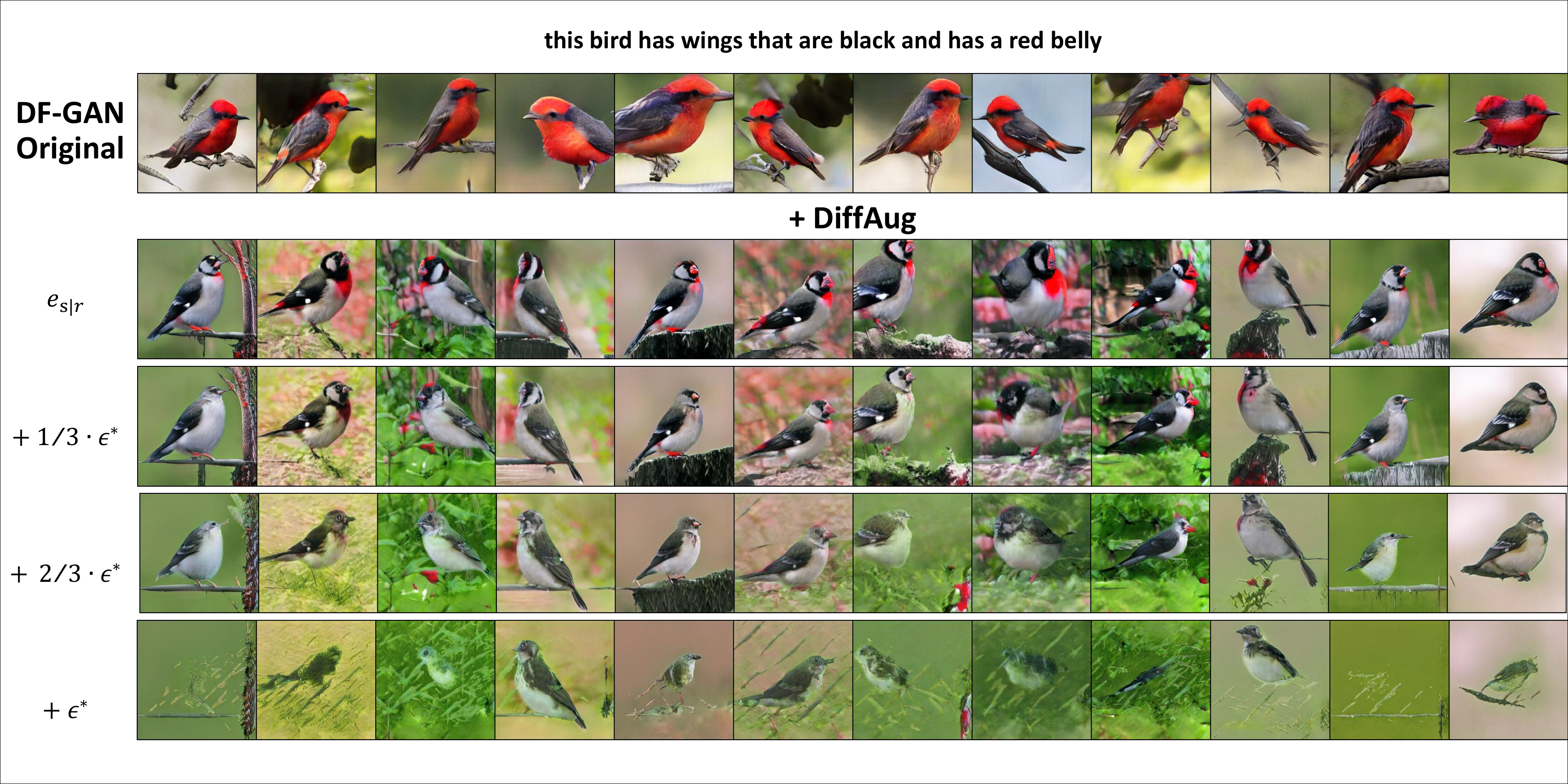}
    \caption{Generated images of DF-GAN with DiffAug on CUB dataset with ascending scales of $\epsilon$ are added.}
    \label{fig:diffaug}
\end{figure*}

\begin{figure*}[t]
    \centering
    \includegraphics[width=0.46\linewidth]{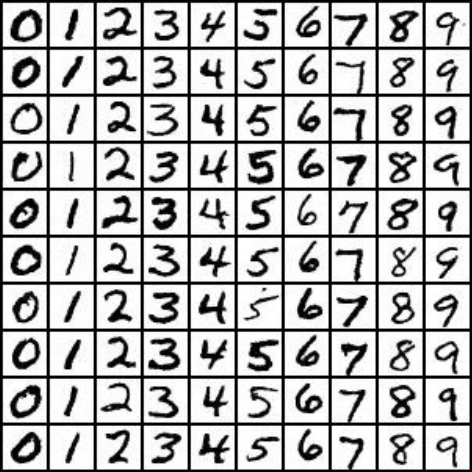} 
    \includegraphics[width=0.46\linewidth]{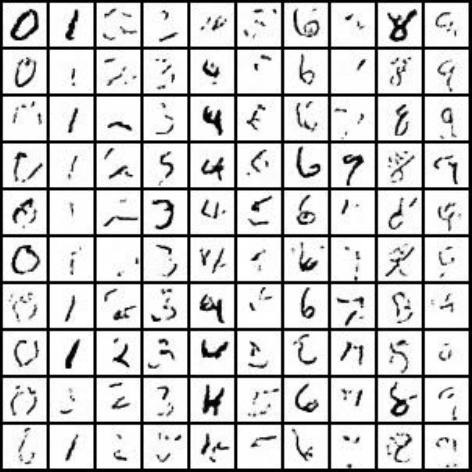}
    \caption{Generated images of DDPM wo/ (left) and w/ (right) perturbations. }
    \label{fig:ddpm_ori}
\end{figure*}

\begin{figure*}[t]
    \centering
    \includegraphics[width=0.46\linewidth] {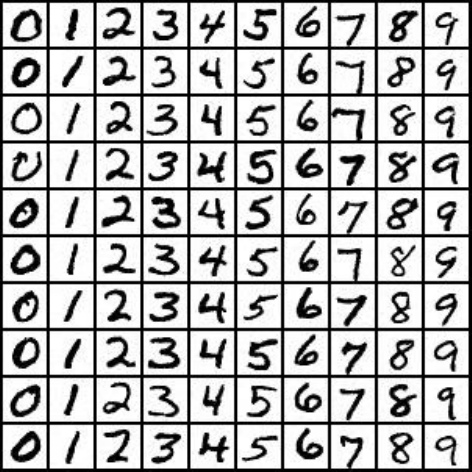} 
    \includegraphics[width=0.46\linewidth]{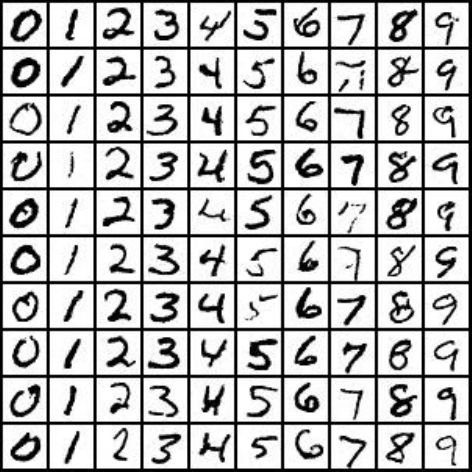}
    \caption{Generated images of DDPM $+ITA_C + L_r$ wo/ (left) and w/ (right) perturbations. }
    \label{fig:ddpm_aug}
\end{figure*}

\end{document}